\documentclass{llncs}

\usepackage{times}
\usepackage{helvet}
\usepackage{courier}
\usepackage{amssymb}

\usepackage{amsmath}
\usepackage{algorithm}
\usepackage[noend]{algpseudocode}
\usepackage{graphicx}
\usepackage{subcaption}
\usepackage[english]{babel}
\usepackage{mathrsfs}
\usepackage{amscd}
\usepackage{stmaryrd}
\usepackage{todonotes}
\usepackage{wrapfig}

\newcommand{\powerset}[1]{\mathcal{P}(#1)}

\newcommand{\be}{\begin{itemize}}
\newcommand{\ee}{\end{itemize}}

\newcommand{\inputdomain}{{\rm D}}
\newcommand{\manipulation}{\delta}

\newcommand{\inputImage}{\alpha}
\newcommand{\distance}[2]{||#2||_{#1}}
\newcommand{\nature}{{\tt nat}}
 \newcommand{\lipschitzConstant}{\hbar}
\newcommand{\classChangeDistance}{\ell}

\usepackage{color}

\title{\textbf{Feature-Guided Black-Box Safety Testing of Deep Neural Networks}}
\author{
  Matthew Wicker\inst{1} \and Xiaowei Huang\inst{2} \and Marta Kwiatkowska\inst{3}
 }
  \institute{
  University of Georgia, USA,
  \email{matthew.wicker25@uga.edu} 
  \and 
  University of Liverpool, UK,
  \email{xiaowei.huang@liverpool.ac.uk} 
   \and 
  University of Oxford, UK,
  \email{marta.kwiatkowska@cs.ox.ac.uk}\\
}

\date{}

\begin{document}
\maketitle
\begin{abstract}

Despite the improved accuracy of deep neural networks, the discovery of adversarial examples has raised serious safety concerns.
Most existing approaches for crafting adversarial examples necessitate some knowledge (architecture, parameters, etc) of the network at hand. In this paper, we focus on image classifiers and propose a \emph{feature-guided} black-box approach 
to test the safety of deep neural networks that requires no such knowledge. 
Our algorithm employs object detection techniques such as SIFT (Scale Invariant Feature Transform) to extract features from an image. These features are converted into a mutable saliency distribution, where high probability is assigned to pixels that affect the composition of the image with respect to the human visual system. We formulate the crafting of adversarial examples as a two-player turn-based stochastic game, where the first player's objective is to minimise the distance to an adversarial example by manipulating the features, and the second player can be cooperative, adversarial, or random. 
We show that, theoretically, the two-player game can converge to the optimal strategy, and that the optimal strategy represents a globally minimal adversarial image.
For Lipschitz networks, we also identify conditions that provide safety guarantees that no adversarial examples exist.
Using Monte Carlo tree search we gradually explore the game state space to search for adversarial examples. 
Our experiments show that, despite the black-box setting, manipulations guided by a perception-based saliency distribution are competitive 
with state-of-the-art methods that rely on white-box saliency matrices or sophisticated optimization procedures. 
Finally, we show how our method can be used to evaluate robustness of neural networks in safety-critical applications such as traffic sign recognition in self-driving cars.

\end{abstract}

\section{Introduction}

\newcommand{\terminate}{t}
\newcommand{\reward}{R}
\newcommand{\playerOne}{{\tt I}}
\newcommand{\playerTwo}{{\tt II}}
\newcommand{\opt}{{\tt opt}}

Deep neural networks (DNNs or networks, for simplicity) have been developed  for 
a variety of tasks, including malware detection \cite{malware}, abnormal network activity detection \cite{ryan:nips10}, and self-driving cars \cite{NVIDIA,road-segmentation,traffic-classification-lecun}. A classification network $N$ can be  used as a decision-making algorithm: given an input $\inputImage$, it suggests a decision $N(\inputImage)$ among a set of possible decisions. While the accuracy of neural networks has greatly improved, matching the cognitive ability of humans~\cite{LBH2015}, they are susceptible to adversarial examples~\cite{Biggio2013,SZSBEGF2014}.  An adversarial example is an input which, though initially classified correctly, is misclassified after a minor, perhaps imperceptible, perturbation. 
Adversarial examples pose challenges for 
self-driving cars, where neural network solutions have been proposed for tasks such as end-to-end steering \cite{NVIDIA}, road segmentation \cite{road-segmentation}, and traffic sign classification \cite{traffic-classification-lecun}. 
In the context of steering and road segmentation, an adversarial example may cause a car to steer off the road or drive into barriers, and misclassifying traffic signs may cause a vehicle to drive into oncoming traffic. 
Fig. \ref{fig:coverimage} shows an image of a traffic light correctly classified by a state-of-the-art network which is then misclassified after only a few pixels have been changed.
Though somewhat artificial, since in practice the controller would rely on additional sensor input whem making a decision,
such cases strongly suggest that, before deployment in safety-critical tasks, DNNs resilience (or robustness) to adversarial examples must be strengthened.
\begin{wrapfigure}{r}{0.6\textwidth}\label{fig:yolo}
    \centering
    
    \includegraphics[width=0.6\textwidth]{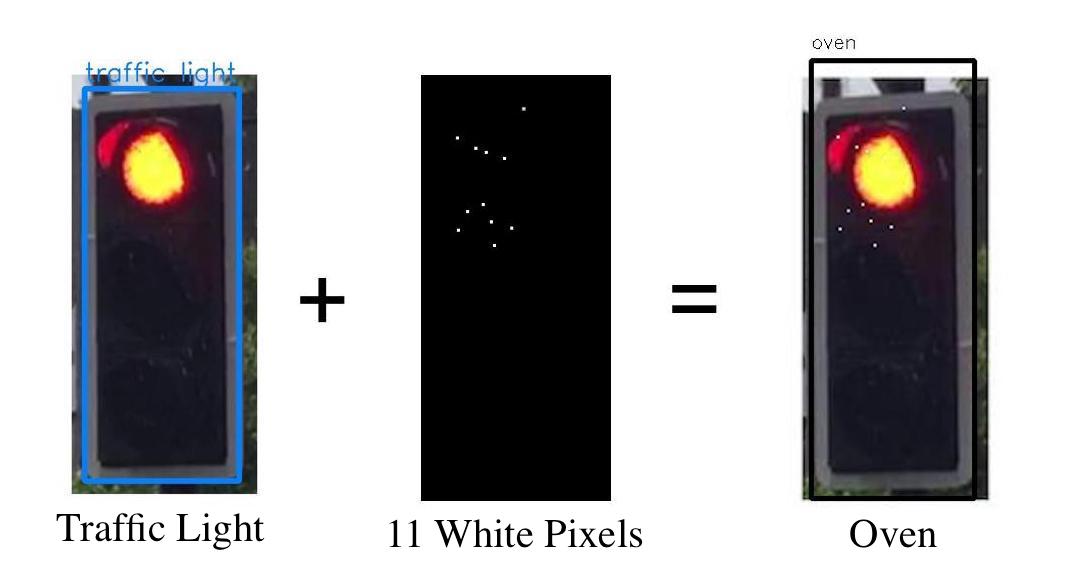}
    
    \caption{An adversarial example for the YOLO object recognition network. 
    }
    \label{fig:coverimage}
\end{wrapfigure}
A number of approaches have been proposed to search for adversarial examples (see Related Work). They are based on computing the
gradients~\cite{FGSM}, along which a heuristic search moves; computing a Jacobian-based saliency map~\cite{JSMA}, based on which pixels are selected to be changed; transforming the existence of adversarial examples into an optimisation problem~\cite{CW-Attacks}, on which an optimisation algorithm can be applied; transforming the existence of adversarial examples into a constraint solving problem~\cite{KBDJK2017}, on which a constraint solver can be applied; or discretising the
neighbourhood of a point and searching it exhaustively in a layer-by-layer manner~\cite{DLV}. 
All these approaches assume some knowledge about the network, e.g., the architecture or the parameters, 
which can vary as the network continuously learns and adapts to new data, and, with a few exceptions \cite{practical-black-box} that access the penultimate layer, do not explore the feature maps of the networks.

In this paper, we propose a \emph{feature-guided} approach to test the resilience of image classifier networks against adversarial examples.
While convolutional neural networks (CNN) have been  successful in classification tasks, their feature extraction capability is not well understood~\cite{YCNFL2015}. 
The discovery of adversarial examples has 
called into question CNN's ability to robustly handle input with diverse structural and compositional elements. 
On the other hand, state-of-the-art feature extraction methods are able to deterministically and efficiently extract structural elements of an image regardless of scale, rotation or transformation. A key observation of this paper is that feature extraction methods  enable us to identify elements of an image which are most vulnerable to a visual system such as a CNN.

{Leveraging knowledge of the human perception system, existing object detection techniques}
{detect instances of semantic objects of a certain class} (such as animals, buildings, or cars) in digital images and videos by identifying their features. We use the scale-invariant feature transform approach, or SIFT~\cite{SIFT}, to detect features, which is achieved with no knowledge of the network in a \emph{black-box} manner. 
Using the SIFT features, whose number is much smaller than the number of pixels, we represent the image as a two-dimensional Gaussian mixture model. This reduction in dimensionality allows us to efficiently target the exploration at salient features, similarly to human perception.
We formulate the process of crafting adversarial examples as a two-player turn-based stochastic game, where player $\playerOne$ selects features and player $\playerTwo$ then selects pixels within the selected features and a manipulation instruction. After both players have made their choices, the image is modified according to the manipulation instruction, 
and the game continues.
While player $\playerOne$ aims to minimise the distance to an adversarial  example, player $\playerTwo$  can be cooperative, adversarial, or nature who samples the pixels according to the Gaussian mixture model. 
We show that, theoretically, the two-player game can converge to the optimal strategy, and that the optimal strategy represents a globally minimal adversarial image. We also consider safety guarantees for Lipschitz networks and identify conditions to ensure that no adversarial examples exist.

We implement a software package\footnote{The software package and all high-resolution figures used in the paper are available from  \url{https://github.com/matthewwicker/SafeCV}}, in which 
a Monte Carlo tree search (MCTS) algorithm is employed to find asymptotically optimal strategies for both  players, with player $\playerTwo$ being a cooperator. The algorithm is 
\emph{anytime}, meaning that it can be terminated with time-out bounds provided by the user and, when terminated, it returns the best strategies it has for both players. The experiments on networks trained on benchmark datasets such as MNIST~\cite{MNIST} and CIFAR10~\cite{cifar10} show that, even without the knowledge of the network and using relatively little time (1 minute for every image), the algorithm can already achieve competitive performance against existing adversarial example crafting algorithms. 
We also experiment on several state-of-the-art networks, including the winner of the Nexar traffic light challenge~\cite{NexarData}, a real-time object detection system YOLO, and VGG16~\cite{VGG16} for ImageNet competition, where, surprisingly, we show that the algorithm can return adversarial examples even with very limited resources (e.g., running time of \emph{less than a second}), including that in Fig.~\ref{fig:coverimage} from YOLO. Further, since the SIFT method is scale and rotation invariant, we can counter claims in the recent paper~\cite{NOneedtoworry} that adversarial examples are not invariant to changes in scale or angle in the physical domain.

Our software package is well suited to safety testing and decision support for DNNs in safety-critical applications. First, the  MCTS algorithm can be used \emph{offline} to evaluate the network's robustness against adversarial examples on a given set of  images. The asymptotic optimal strategy achievable by MCTS algorithm enables a theoretical guarantee of safety, i.e., the network is safe when the algorithm cannot find adversarial examples. 
The algorithm is guaranteed to terminate, but this may be impractical, so we provide an alternative termination criterion.
Second, 
the MCTS algorithm, in view of its time efficiency, has the potential to be deployed on-board for \emph{real-time} decision support.

An extended version of the paper, which includes more additional explanations and experimental results, is available from \cite{WHK2017}.

\section{Preliminaries}\label{sec:preliminaries}

\newcommand{\severity}{sev}
\newcommand{\expectation}{{\tt E}}
\newcommand{\GaussianMixture}{{\cal G}}
\newcommand{\feature}{\lambda}
\newcommand{\setoffeatures}{\Lambda}
\newcommand{\strategy}{\sigma}

Let $N$ be a network with a set $C$ of classes. Given an input $\inputImage$ and a class $c \in C$, we use $N(\inputImage,c)$ to denote the confidence (expressed as a probability value obtained from normalising the score) of $N$ believing  that $\inputImage$ is in class $c$. Moreover, we write $N(\inputImage) = \arg\max_{c\in C} N(\inputImage,c)$ for the class into which $N$ classifies $\inputImage$. 
For our discussion of image classification networks, the input domain $\inputdomain$ is a vector space, which in most cases can be represented as 
${\rm I\!R_{[0,255]}^{w\times h\times ch}}$, where $w,h,ch$ are the width, height, and number of channels of an image, respectively, and we let $P_0 = w\times h\times ch$ be the set of input dimensions. 
In the following, we may refer to an element in $w\times h$ as a pixel and an element in $P_0$ as a dimension. 
We remark that dimensions are normalised as real values in $[0,1]$. Image classifiers employ a distance function to compare images. Ideally, such a distance should reflect perceptual similarity between images, comparable to human perception. However, in practice $L_k$ distances are used instead,
typically
$L_0$, $L_1$ (Manhattan distance), $L_2$ (Euclidean distance), and $L_\infty$ (Chebyshev distance). 
We also work with $L_k$ distances but emphasise that our method can be adapted to other distances.
In the following, we write $\distance{k}{\inputImage_1-\inputImage_2}$ with $k\geq 0$ for the distance between two images $\inputImage_1$ and $\inputImage_2$ with respect to the $L_k$ measurement.

Given an image $\inputImage$, a distance measure $L_k$, and a distance $d$, we define $\eta(\inputImage,k,d)=\{\inputImage' ~|~ \distance{k}{\inputImage'-\inputImage} \leq d\}$
as the set of points whose distance to $\inputImage$ is no greater than $d$ with respect to $L_k$.  
Next we define adversarial examples, as well as what we mean by targeted and non-targeted safety.

\begin{definition}\label{def:constraints}
Given an input $\inputImage\in\inputdomain$, a distance measure $L_k$ for some $k\geq 0$, and a distance $d$, an \emph{adversarial example} $\inputImage'$ of class $c\neq N(\inputImage)$ is such that 
$\inputImage'\in \eta(\inputImage,k,d)$, 
$N(\inputImage)\neq N(\inputImage')$, and 
$N(\inputImage')=c$. 
Moreover, we write $adv_{N,k,d}(\inputImage,c)$ for the set of adversarial examples of class $c$ and let $adv_{N,k,d}(\inputImage)=\bigcup_{c\in C, c\neq N(\inputImage)}adv_{N,k,d}(\inputImage,c)$. A \emph{targeted safety} of class $c$ is defined as $adv_{N,k,d}(\inputImage,c)=\emptyset$, and a \emph{non-targeted safety} is defined as $adv_{N,k,d}(\inputImage)=\emptyset$.  

\end{definition}

\smallskip\noindent{\bf Feature Extraction}
The Scale Invariant Feature Transform (SIFT) algorithm \cite{SIFT}, 
a reliable technique for exhuming features from an image, makes object localization and tracking possible without the use of neural networks. Generally, the SIFT algorithm proceeds through the following steps: scale-space extrema detection (detecting relatively darker or lighter areas in the image), keypoint localization (determining the exact position of these areas), and keypoint descriptor assignment (understanding the context of the image w.r.t its local area).
\begin{figure}[h]
    \centering
    \includegraphics[width=\textwidth]{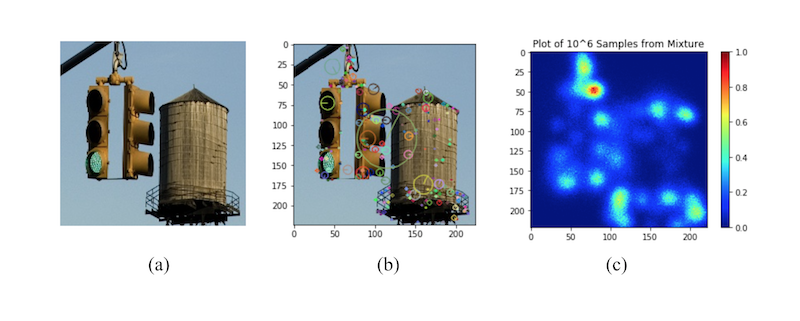}
    \caption{Illustration of the transformation of an image into a saliency distribution. (a) The original image $\inputImage$, provided by ImageNet. (b) The image marked with relevant keypoints $\setoffeatures(\inputImage)$. 
    (c) The heatmap of the Gaussian mixture model $\GaussianMixture(\setoffeatures(\inputImage))$.}
    \label{fig:imtosal2}

\end{figure}
\newcommand{\LKdistance}{D_{\text{KL}}}
Human perception of an image or an object can be reasonably represented as a set of features (referred to as keypoints in SIFT) of different sizes and response strengths, see \cite{Szeliski2010} and Appendix of \cite{WHK2017} for more detail.  
Let $\setoffeatures(\inputImage)$ be a set of features  of the image $\inputImage$ such that each feature $\feature\in \setoffeatures(\inputImage)$ is a tuple $(\feature_x,\feature_y,\feature_s,\feature_r)$, where $(\lambda_x,\lambda_y)$ is the coordinate of the feature in the image, $\feature_s$ is the size of the feature, and $\feature_r$ is the response strength of the feature. 
The SIFT procedures implemented in standard libraries such as OpenCV may return more information which we do not use.

{On their own, keypoints are not guaranteed to involve every pixel in the image, and in order to ensure a comprehensive and flexible safety analysis, we utilize these keypoints as a basis for a Guassian mixture model.}
Fig.~\ref{fig:imtosal2} shows the original image (a) and this image annotated with keypoints (b).

\smallskip\noindent{\bf Gaussian Mixture Model}
Given an image $\inputImage$ and its set $\setoffeatures(\inputImage)$ of keypoints, we define for $\lambda_i\in \setoffeatures(\inputImage)$
a two-dimensional Gaussian distribution $\GaussianMixture_i$ such that, for pixel $(p_x,p_y)$, we have 
\begin{equation}
\GaussianMixture_{i,x} = \dfrac{1}{\sqrt{2\pi\lambda_{i,s}^{2}}}                exp\big(\dfrac{-(p_x-\lambda_{i,x})^2}{2\lambda_{i,s}^{2}}\big)~~~~~
\GaussianMixture_{i,y} = \dfrac{1}{\sqrt{2\pi\lambda_{i,s}^{2}}}                exp\big(\dfrac{-(p_y-\lambda_{i,y})^2}{2\lambda_{i,s}^{2}}\big) \label{equation:gaussian}
\end{equation}
where the variance is the size $\lambda_{i,s}$ of the keypoint and the mean is its location $(\lambda_{i,x}, \lambda_{i,y})$. To complete the model, we define a set of weights $\Phi = \{\phi_i\}_{ i \in \{1,2,...,k\} }$ such that $k=|\setoffeatures(\inputImage)|$ and 
$\phi_i = \lambda_{i,r}/\sum_{j=0}^{k}\lambda_{j,r} $.
Then, we can construct a Gaussian mixture model $\GaussianMixture$ by combining the distribution components with the weights as coefficients, i.e., 
$\GaussianMixture_{x} = \prod_{i=1}^k \phi_i\times \GaussianMixture_{i,x}$ and 
$\GaussianMixture_{y} = \prod_{i=1}^k \phi_i\times \GaussianMixture_{i,y}$.
{The two-dimensional distributions 
are discrete and separable and therefore their realization is tractable and independent, which improves efficiency of computation.}
Let $\GaussianMixture(\setoffeatures(\inputImage))$ be the obtained Gaussian mixture model from $\setoffeatures(\inputImage)$, and $G$ be the set of Gaussian mixture models.
In Fig.~\ref{fig:imtosal2} we illustrate the transformation of an image into a saliency distribution.

\newcommand{\instruction}{i}
\newcommand{\instructionset}{I}

\smallskip\noindent{\bf Pixel Manipulation}
We now define the operations that we consider for manipulating images.
We write $\inputImage(x,y,z)$ for the value of the $z$-channel {(typically RGB or grey-scale values)} of the pixel positioned at $(x,y)$ on the image $\inputImage$. Let $\instructionset=\{+,-\}$ be a set of manipulation instructions and $\tau$ be a positive real number representing the manipulation magnitude, then we can define pixel manipulations $\manipulation_{X,\instruction}: \inputdomain \rightarrow \inputdomain$ for $X$ a subset of input pixels  
and $\instruction\in \instructionset$:  
$$
\manipulation_{X,\instruction}(\inputImage)(x,y,z) = \left\{
\begin{array}{ll}
\inputImage(x,y,z) + \tau, & \text{if } (x,y)\in X \text{ and } \instruction = + \\
\inputImage(x,y,z) - \tau, & \text{if } (x,y)\in X \text{ and } \instruction = - \\
\inputImage(x,y,z)  & \text{otherwise}\\
\end{array}
\right.
$$
for all pixels $(x,y)$ and channels $z\in \{1,2,3\}$.
Note that if the values are bounded, e.g., $[0,1]$,  $\manipulation_{X,\instruction}(\inputImage)(x,y,z)$ needs to be restricted to be within the bounds. For simplicity, in our experiments and comparisons we allow a manipulation to choose either  the upper bound or the lower bound with respect to the instruction $\instruction$. For example, in Fig.~\ref{fig:coverimage}, the actual  manipulation considered is to make the manipulated dimensions choose value 1.

\section{Safety Against Manipulations}
\label{SafetyThms}

Recall that every image represents a point in the input vector space $\inputdomain$. Most existing investigations of the safety (or robustness) of DNNs 
focus on optimising the movement of a point along the gradient direction of some function obtained from the network (see Related Work for more detail). Therefore, these approaches rely on the knowledge about the DNN. Arguably, this reliance holds also for the black-box approach proposed in \cite{practical-black-box}, which uses  
a new surrogate network trained on the  data sampled from the original network. 
Furthermore, the current understanding about the transferability of adversarial examples (i.e., an adversarial example found for a network can also serve as an adversarial example for another network, trained on different data) are all based on empirical experiments~\cite{practical-black-box}. The conflict between the understanding of transferability and  existing  approaches to crafting adversarial examples can be gleaned from an observation made in \cite{LCLS2017} that gradient directions of different models are orthogonal to each other. A reasonable interpretation is that transferable adversarial examples, if they exist, do not rely on the gradient direction suggested by a network but instead may be specific to the input. 

In this paper, we propose a \emph{feature-guided} 
approach 
which, instead of using the gradient direction as the guide for optimisation, relies on searching fro adversarial examples by targeting and manipulating image features as recognised by human perception capability.
We extract features using SIFT, which is a reasonable proxy for human perception and enables dimensionality reduction through the Gaussian mixture representation (see \cite{Reynolds2009GaussianMM}).
Our method
needs neither the knowledge about the network nor the necessity to massively sample the network for data to train a new network, and is therefore a \emph{black-box} approach.

\smallskip\noindent{\bf Game-based Approach}
We formulate the search for adversarial examples as a two-player turn-based stochastic game, where player $\playerOne$ selects features and player $\playerTwo$ then selects pixels within the selected features and a manipulation instruction. 
While player $\playerOne$ aims to minimise the distance to an adversarial  example, player $\playerTwo$  can be cooperative, adversarial, or nature who samples the pixels according to the Gaussian mixture model. 
To give more intuition for feature-guided search, in Appendix of~\cite{WHK2017} 
we demonstrate how the  distribution of the Gaussian mixture model representation evolves for different adversarial examples. 

We define the objective function in terms of the $L_k$ distance and view the distance to an adversarial example as a measure of its severity. Note that the sets $adv_{N,k,d}(\inputImage,c)$ and $adv_{N,k,d}(\inputImage)$ of adversarial examples can be infinite.

\begin{definition} \label{def:objective}
  Among all adversarial examples in the set $adv_{N,k,d}(\inputImage,c)$  (or $adv_{N,k,d}(\inputImage)$), find $\inputImage'$ with the minimum distance to the original image $\inputImage$: 
\begin{equation}\label{equ:objective}
\arg\min_{\inputImage'} \{ \severity_\inputImage(\inputImage')~|~\inputImage' \in adv_{N,k,d}(\inputImage,c) (\text{or } adv_{N,k,d}(\inputImage))\}
\end{equation}
where $\severity_\inputImage(\inputImage') = \distance{k}{\inputImage - \inputImage'}$ is the \emph{severity} of the adversarial example $\inputImage'$ against the original image $\inputImage$.
\end{definition}

\noindent
We remark that the choice of $L_k$ will affect perceptual similarity, see Appendix of~\cite{WHK2017}.

\smallskip\noindent{\bf Crafting Adversarial Examples as a Two-Player Turn-Based Game}
Assume two players $\playerOne$ and $\playerTwo$.  
Let $M(\inputImage,k,d)=(S\cup (S\times \setoffeatures(\inputImage)),s_0,\{T_a\}_{a \in \{\playerOne,\playerTwo\}},L)$ be a game model, where $S$ is a set of game states belonging to player $\playerOne$ such that each state represents an image in $ \eta(\inputImage,k,d)$, and $S\times \setoffeatures(\inputImage)$ is a set of game states belonging to player $\playerTwo$ where 
$\setoffeatures(\inputImage)$ is a set of features (keypoints) of image $\inputImage$. We write $\inputImage(s)$ for the image associated to the state $s\in S$. $s_0\in S$ is the initial game state such that $\inputImage(s_0)$ is the original image $\inputImage$. 
The transition relation 
$T_\playerOne: S \times\setoffeatures(\inputImage)\rightarrow S \times\setoffeatures(\inputImage)$ is defined as  $T_\playerOne(s,\feature)=(s,\feature)$, and 
transition relation 
$T_\playerTwo: (S \times\setoffeatures(\inputImage))\times \powerset{P_0}\times \instructionset\rightarrow S$ is defined as $T_\playerTwo((s,\feature),X,\instruction)=\manipulation_{X,i}(\inputImage(s))$, where $\manipulation_{X,i}$  is a pixel manipulation defined in Section~\ref{sec:preliminaries}.
Intuitively, on every game state $s\in S$, player $\playerOne$ will choose a keypoint $\feature$, and, in response to this, player $\playerTwo$ will choose a pair $(X,\instruction)$,  where $X$ is a set of input dimensions and $\instruction$ is a manipulation instruction.  The labelling function $L:S\cup (S\times \setoffeatures(\inputImage))\rightarrow C\times G$ assigns to each state $s$ or $(s,\feature)$ a class $N(\inputImage(s))$ and a two-dimensional Gaussian mixture model $\GaussianMixture(\setoffeatures(\inputImage(s)))$. 

A path (or game play) of the game model is a sequence $s_1u_1s_2u_2...$ of game states such that,  for all $k\geq 1$, we have $u_k = T_{\playerOne}(s_k,\feature_k)$ for some feature $\feature_k$ and $s_{k+1}=T_\playerTwo((s_k,\feature_k),X_k,\instruction_k)$ for some $(X_k,\instruction_k)$. Let $last(\rho)$ be the last state of a finite path $\rho$ and $Path_a^F$ be the set of finite paths such that $last(\rho)$ belongs to player $a\in \{\playerOne,\playerTwo\}$. 
A stochastic strategy $\strategy_\playerOne: Path_\playerOne^F\rightarrow \mathcal{D}(\setoffeatures(\inputImage))$ of player $\playerOne$ maps each finite paths to a distribution over the next actions, and similarly for 
$\strategy_\playerTwo:Path_\playerTwo^F\rightarrow \mathcal{D}(\powerset{P_0}\times \instructionset)$ for  player $\playerTwo$.
We call $\strategy = (\strategy_\playerOne,\strategy_\playerTwo)$ a strategy profile. In this section, we only discuss targeted safety for a given target class $c$ (see Definition~\ref{def:constraints}). All the notations and results can be easily adapted to work with non-targeted safety. 

In the following, we define a reward $\reward(\strategy,\rho)$ for a given strategy profile $\strategy = (\strategy_\playerOne,\strategy_\playerTwo)$ and a finite path $\rho\in \bigcup_{a\in\{\playerOne,\playerTwo\}}Path_a^F$. 
The idea of the reward is to accumulate a measure of severity of the adversarial example found over a path. 
Note that, given $\strategy$, the game becomes a fully probabilistic system. 
Let $\inputImage_\rho' = \inputImage(last(\rho))$  be  the image associated with the last state of the path $\rho$. We write $\terminate(\rho)$ for the expression $N(\inputImage_\rho')=c \lor \distance{k}{ \inputImage_\rho'-\inputImage} > d$, representing that the path has reached a state whose associated image either is in the target class $c$ or lies outside the region $ \eta(\inputImage,k,d)$. The path $\rho$ can be  terminated whenever $\terminate(\rho)$ is satisfiable. It is not hard to see that, due to the constraints in Definition~\ref{def:constraints}, every infinite path has a finite prefix which can be terminated. Then we define the reward function $\reward(\strategy,\rho) = $
$$
\left\{
\begin{array}{lll}
1/ \severity_\inputImage(\inputImage_\rho') & \text{if } \terminate(\rho) \text{ and }\rho\in Path_\playerOne^F \\
\sum_{\feature \in \setoffeatures(\inputImage)} \strategy_\playerOne(\rho)(\feature) \cdot \reward(\strategy,\rho T_\playerOne(last(\rho),\feature)) & \text{if } \neg\terminate(\rho) \text{ and }\rho\in Path_\playerOne^F \\
 \sum_{(X,\instruction)\in \powerset{P_0}\times \instructionset} \strategy_\playerTwo(\rho )(X,\instruction) \cdot \reward(\strategy,\rho T_\playerTwo(last(\rho),X,\instruction)) & \text{if } \rho\in Path_\playerTwo^F
\end{array} 
\right. 
$$  \label{rewardFunction}
\noindent
where 
$\strategy_\playerOne(\rho)(\feature)$ is the probability of selecting $\feature$ on $\rho$ by player $\playerOne$, and $\strategy_\playerTwo(\rho)(X,\instruction)$ is the probability of selecting $(X,\instruction)$ based on $\rho$ by player $\playerTwo$. We note that a path only terminates on player $\playerOne$ states. 

Intuitively, if an adversarial example is found then the reward assigned is the inverse of severity (minimal distance), and otherwise it is the weighted summation of the rewards if its children.  Thus, a strategy $\strategy_\playerOne$ to maximise the reward will need to minimise the severity  $\severity_\inputImage(\inputImage_\rho')$, the objective of the problem defined in Definition~\ref{def:objective}. 

\begin{definition}
The goal of the game is for player $\playerOne$ to choose a strategy $\strategy_{\playerOne}$ to maximise  the reward $\reward((\strategy_\playerOne,\strategy_\playerTwo),s_0) $ of the initial state $s_0$, based on the strategy $\strategy_{\playerTwo}$ of the player $\playerTwo$, i.e.,
\begin{equation}
\arg\max_{\strategy_{\playerOne}} \opt_{\strategy_{\playerTwo}} \reward((\strategy_\playerOne,\strategy_\playerTwo),s_0). 
\end{equation}
where option $\opt_{\strategy_{\playerTwo}}$ can be $\max_{\strategy_{\playerTwo}} $, $\min_{\strategy_{\playerTwo}} $, or $\nature_{\strategy_{\playerTwo}}$, according to which player $\playerTwo$ acts as a cooperator,  an adversary, or nature who samples the distribution $\GaussianMixture(\setoffeatures(\inputImage))$ for pixels and randomly chooses the manipulation instruction. 
\end{definition}

A strategy $\sigma$ is called deterministic if $\sigma(\rho)$ is a Dirac distribution, and is called memoryless if $\sigma(\rho)=\sigma(last(\rho))$ for all finite paths $\rho$. We have the following result. 
\begin{theorem}\label{sec:det}
Deterministic and memoryless strategies suffice for player $\playerOne$, when $\opt_{\strategy_{\playerTwo}} \in \{\max_{\strategy_{\playerTwo}}, \min_{\strategy_{\playerTwo}}, \nature_{\strategy_{\playerTwo}}\}$. 
\end{theorem}

\smallskip\noindent{\bf Complexity of the Problem}
As a by-product of Theorem~\ref{sec:det}, the theoretical complexity of the problem (i.e., determining whether $adv_{N,k,d}(\inputImage,c)=\emptyset$) is in PTIME, with respect to the size of the game model $M(\inputImage,k,d)$. However, even if we only consider finite paths (and therefore a finite system), the number of states (and therefore the size of the system) is $O(|P_0 |^{h})$ for $h$ the length of the longest finite path of the system without a terminating state. While the precise size of $O(|P_0|^{h})$ is dependent on the problem (including the image $\inputImage$ and the difficulty of crafting an adversarial example), it is roughly $O(50000^{100})$ for the images used in the ImageNet competition and $O(1000^{20})$ for smaller images such as CIFAR10 and MNIST. This is beyond the capability of existing approaches for exact or $\epsilon$-approximate computation of probability (e.g., reduction to linear programming, value iteration, and policy iteration, etc) that are used in probabilistic verification.

\section{Monte Carlo Tree Search for Asymptotically Optimal Strategy}\label{sec:mcts}

\newcommand{\update}{U\!pdate}
\newcommand{\pathsample}{Simulation}
\newcommand{\confidence}{con\!f}
\newcommand{\normalisation}{N\!orm}
\newcommand{\selection}{selection}
\newcommand{\expansion}{expansion}
\newcommand{\backPropogation}{backPropogation}

In this section, we present an approach based on Monte Carlo tree search (MCTS)~\cite{CWU2008} to find an optimal strategy asymptotically. We also we show that the optimal strategy, if achieved, represents the best adversarial example with respect to the objective  in Definition~\ref{def:objective}, under some conditions. 

We first consider the case of $\opt_{\strategy_{\playerTwo}}=\max_{\strategy_{\playerTwo}}$.
An MCTS algorithm, whose pseudo-code is presented in Algorithm~\ref{MCTS}, gradually expands a \emph{partial game tree} by sampling the strategy space of the model $M(\inputImage,k,d)$. With the upper confidence bound (UCB)~\cite{KS2006} as the exploration-exploitation tradeoff, MCTS has a theoretical guarantee that it converges to optimal solution when the game tree is fully explored.  
The algorithm mainly follows the standard MCTS procedure, with a few adaptations. We use two termination conditions $tc_1$
and $tc_2$ to control the pace of the algorithm. More specifically, $tc_1$ controls whether the entire procedure should be terminated, and $tc_2$ controls when a move should be made. The terminating conditions  can be, e.g., bounds on the number of iterations, etc. 
On the partial tree, every node maintains a pair $(r,n)$, which represents the accumulated reward $r$ and the number of visits $n$, respectively. 
The $\selection$ procedure travels from the root to a leaf according to an exploration-exploitation balance, i.e., UCB~\cite{KS2006}. {After $expanding$ the leaf node to have its children added to the partial tree}, we call  $\pathsample$ to run simulation on every child node. 
{A simulation on a new $node$ is a  play of the game from $node$ until it terminates. }
Players act randomly during the simulation. Every simulation terminates when reaching a terminated node $\inputImage'$, on which a reward $1/ \severity_\inputImage(\inputImage')$  can be computed. This reward is then $backpropagated$ from the new child node through its ancestors until reaching the root. 
Every time a new reward $v$ is backpropogated through a node, we update its associated pair to $(r+v,n+1)$. The $bestChild(root)$ returns the child of $root$ which has the highest value of $r/n$. 
\begin{algorithm}
\small
\caption{Monte Carlo Tree Search for $\opt_{\strategy_{\playerTwo}}=\max_{\strategy_{\playerTwo}}$}\label{MCTS}
\begin{algorithmic}[1]
\State \textbf{Input:} A game model $M(\inputImage,k,d)$, two termination conditions $tc_1$ and $tc_2$, a target class $c$
\State \textbf{Output:} An adversarial  example $\inputImage'$
\Procedure{MCTS($M(\inputImage,k,d), tc_1, tc_2, c$)}{}
\State $root \gets s_0$
\State \textbf{While}$(\neg tc_1)$: 
\State \quad \textbf{While}$(\neg tc_2)$: 
\State \quad \quad $leaf \gets \selection(root) $ 
\State \quad \quad $newnodes  \gets \expansion(M(\inputImage,k,d), leaf)$ 
\State \quad \quad \textbf{for} $node$ in $newnodes$: 
\State \quad \quad \quad $v \gets \pathsample(M(\inputImage,k,d),node,c) $
\State \quad \quad \quad $\backPropogation(node, v) $
\State \quad $root \gets bestChild(root)$
\State return $root$
\EndProcedure
\end{algorithmic}
\end{algorithm}
The other two cases are similar except for the choice of the next move (i.e., Line 12). Instead of choosing the best child, a child is chosen by sampling $\GaussianMixture(\setoffeatures(\inputImage))$ for the case of $\opt_{\strategy_{\playerTwo}}=\nature_{\strategy_{\playerTwo}}$, and the worst child is chosen for the case of $\opt_{\strategy_{\playerTwo}}=\min_{\strategy_{\playerTwo}}$. 
{We remark the game is not zero-sum when $\opt_{\strategy_{\playerTwo}}\in \{\nature_{\strategy_{\playerTwo}},\max_{\strategy_{\playerTwo}}\}$.}

\smallskip\noindent{\bf Severity Interval from the Game}
Assume that we have fixed termination conditions $tc_1$ and $tc_2$ and target class $c$. 
Given an option $\opt_{\strategy_{\playerTwo}}$ for player $\playerTwo$, we have an MCTS algorithm to compute an adversarial example $\inputImage'$. Let 
$\severity(M(\inputImage,k,d),\opt_{\strategy_{\playerTwo}})$ be $sev_{\inputImage}(\inputImage')$, where $\inputImage'$ is the returned adversarial example by running Algorithm \ref{MCTS} over the inputs $M(\inputImage,k,d)$, $tc_1$, $tc_2$, $c$ for a certain $\opt_{\strategy_{\playerTwo}}$. 
Then there exists a  severity interval $SI(\inputImage,k,d)$ with respect to the role of player $\playerTwo$: 
\begin{equation}
[\severity(M(\inputImage,k,d), \max_{\strategy_{\playerTwo}}), ~~\severity(M(\inputImage,k,d), \min_{\strategy_{\playerTwo}}) ].
\end{equation}
Moreover, we have that $\severity(M(\inputImage,k,d), \nature_{\strategy_{\playerTwo}})\in SI(\inputImage,k,d)$.

\newcommand{\hyperRectangle}{rec}
\newcommand{\direction}{\gamma}
\newcommand{\directions}{\Gamma}
\newcommand{\tauimage}{G}

\smallskip\noindent{\bf Safety Guarantee via Optimal Strategy}
Recall that $\tau$, a positive real number, is the manipulation magnitude used in pixel manipulations.
An  image $\inputImage' \in \eta(\inputImage,k,d)$ is a $\tau$-grid image if for all dimensions $p\in P_0$ we have $|\inputImage'(p)-\inputImage(p)| = n * \tau$ for some $n\geq 0$. Let $\tauimage(\inputImage,k,d)$ be the set of $\tau$-grid images in $\eta(\inputImage,k,d)$. 
First of all, we have the following conclusion for the case when player $\playerTwo$ is cooperative. 
\begin{theorem}\label{thm:taugrid}
Let $\inputImage' \in \eta(\inputImage,k,d)$ be any $\tau$-grid image such that $\inputImage' \in adv_{N,k,d}(\inputImage,c)${, where $c$ is the targeted class}. Then we have  that $ \severity_\inputImage(\inputImage') \geq \severity(M(\inputImage,k,d), \max_{\strategy_{\playerTwo}})$. 
\end{theorem}

{Intuitively, the theorem says that the algorithm can find the optimal adversarial example from the set of $\tau$-grid images.} 
The idea of the proof is to show that every  $\tau$-grid image can be reached by some game play. {In the following, we show that, if the network is Lipschitz continuous, we need only consider $\tau$-grid images when $\tau$ is small enough. Then, together with the above theorem, we can conclude that our algorithm is both sound and complete. }

Further, we say that an image $\inputImage_1\in \eta(\inputImage,k,d)$ is a misclassification aggregator with respect to a number $\beta>0$ if,  for any $\inputImage_2\in \eta(\inputImage_1,1,\beta)$, we have that $N(\inputImage_2) \neq N(\inputImage)$ implies $N(\inputImage_1) \neq  N(\inputImage)$. {Intuitively, if a misclassification aggregator $\inputImage_1$ with respect to $\beta$ is classified correctly then all input images in $\eta(\inputImage_1,1,\beta)$ are classified correctly. We remark that the region $\eta(\inputImage_1,1,\beta)$ is defined with respect to the $L_1$ metric, but can also be defined using $L_{k'}$, some $k'$, without affecting the results if  $\eta(\inputImage,k,d) \subseteq \bigcup_{\inputImage_1\in \tauimage(\inputImage,k,d)}\eta(\inputImage_1,k',\tau/2)$}. Then we have the following theorem.

\begin{theorem}\label{thm:misclassification}
If all $\tau$-grid images are misclassification aggregators with respect to $\tau/2$, and $\severity(M(\inputImage,k,d), \max_{\strategy_{\playerTwo}}) > d$, then $adv_{N,k,d}(\inputImage,c)=\emptyset$. 
\end{theorem}

\noindent
{Note that $\severity(M(\inputImage,k,d), \max_{\strategy_{\playerTwo}}) > d$ means that none of  the $\tau$-images in $\eta(\inputImage,k,d)$ is an adversarial example. } The theorem suggests that, to achieve a complete safety verification, one may gradually decrease $\tau$ until either $\severity(M(\inputImage,k,d), \max_{\strategy_{\playerTwo}}) \leq d$, in which case we claim the network is unsafe,  or the condition that all $\tau$-grid images are misclassification aggregators with respect to $\tau/2$ is satisfiable, in which case we claim the network is safe. 
In the following, we discuss how to decide the largest $\tau$ for a Lipschitz network, in order to satisfy that condition and therefore achieve a complete verification using our approach. 

\begin{definition}
Network $N$ is a Lipschitz network with respect to the distance $L_k$ and a constant $\lipschitzConstant > 0 $ if, for all $\inputImage,\inputImage'\in \inputdomain$, we have $|N(\inputImage', N(\inputImage))-N(\inputImage, N(\inputImage))| < \lipschitzConstant \cdot \distance{k}{\inputImage' - \inputImage} $. 
\end{definition}

\noindent
Note that all networks whose inputs are bounded, including all image classification networks we studied, are Lipschitz networks. {Specifically, it is shown in~\cite{RHK2018} that most known types of layers, including fully-connected, convolutional, ReLU, maxpooling, sigmoid, softmax, etc., are Lipschitz continuous.} Moreover, we let $\classChangeDistance$ be the minimum confidence gap for a class change, i.e., 
$$\classChangeDistance = \min \{|N(\inputImage', N(\inputImage))-N(\inputImage, N(\inputImage))| ~~|~~  \inputImage,\inputImage'\in\inputdomain, 
 N(\inputImage')\neq N(\inputImage) \}.
$$
The value of $\classChangeDistance$ is in $[0,1]$, dependent on the network, and can be {estimated by examining all input examples $\inputImage'$ in the training and test data sets, or computed  with provable guarantees by  reachability analysis~\cite{RHK2018}. }
{The following theorem can be seen as an instantiation of Theorem~\ref{thm:misclassification} by using Lipschitz continuity with  $\tau \leq  \frac{2 \classChangeDistance}{ \lipschitzConstant}$ to implement the misclassification aggregator. }

\begin{theorem}
Let $N$ be a Lipschitz network with respect to $L_1$ and a constant $\lipschitzConstant $. Then, when $\tau \leq  \frac{2 \classChangeDistance}{ \lipschitzConstant}$ and $\severity(M(\inputImage,k,d), \max_{\strategy_{\playerTwo}}) > d$, we have that $adv_{N,k,d}(\inputImage,c)=\emptyset$. 
\end{theorem}

\smallskip\noindent{\bf $1/\epsilon$-convergence} 
Because we are working with a finite game, MCTS is guaranteed to converge when the game tree is fully expanded. In the worst case, it may take a very long time to converge. In practice, we can work with $1/\epsilon$-convergence by letting the program terminate when the current best adversarial example has  not been improved by finding a less severe one for  $\lceil 1/\epsilon \rceil$  iterations, where $\epsilon>0$ is a small real number.

\section{Experimental Results}\label{sec:mcexp}

For our experiments, we let player $\playerTwo$ be a cooperator, and its move $(X,\instruction)$ is such that for all $(x_1,y_1,z_1),(x_2,y_2,z_2) \in X$ we have $x_1=x_2$ and $y_1=y_2$, i.e., one pixel (including 3 dimensions for color images or 1 dimension for grey-scale images) is changed for every move. When running simulations (Line 10 of Algorithm~\ref{MCTS}), we let $\strategy_\playerOne(\feature)=\feature_r/\sum_{\feature\in\setoffeatures(\inputImage)}\feature_r$ for all keypoints $\feature\in\setoffeatures(\inputImage)$ and $\opt_{\strategy_{\playerTwo}}=\nature_{\strategy_{\playerTwo}}$. That is, player $\playerOne$ follows a stochastic strategy to choose a keypoint according to its response strength and player $\playerTwo$ is nature. In this section, we compare our method with existing approaches, show convergence of the MCTS algorithm on limited runs, evaluate safety-critical networks trained on traffic light images, and counter-claim a recent statement regarding adversarial examples in physical domains.

\smallskip\noindent{\bf Comparison with Existing Approaches}
We compare our approach to two state-of-the-art methods on two image classification networks, trained on the well known benchmark datasets MNIST and CIFAR10. 
The MNIST image dataset contains images of size $28\times 28$ and one channel and the network is trained with the source code given in~\cite{mnistNetwork}. The trained network is of medium size with 600,810 real-valued parameters, and achieves state-of-the-art accuracy, exceeding 99\%. It has 12 layers, within which there are 2 convolutional layers, as well as layers such as ReLU, dropout, fully-connected layers and a softmax layer. 
The CIFAR10 dataset contains small images, $32\times 32$, with three channels, and the network is trained with the source code from~\cite{cifar10}  for more than 12 hours. The trained network has 1,250,858 real-valued parameters and includes convolutional layers, ReLU layers, max-pooling layers, dropout layers, fully-connected layers, and a softmax layer.
For both networks, the images are preprocessed to make the value of each dimension lie within the bound $[0,1]$.
We randomly select 1000 images $\{\inputImage_i\}_{i\in \{1..1000\}}$ from both datasets for non-targeted safety testing. The numbers  in Table~\ref{tab:mcts} are average distances defined as 
$\frac{1}{1000}\cdot \sum_{i=1}^{1000}\distance{0}{\inputImage_i - \inputImage_i'}$,
where $\inputImage_i'$ is the adversarial image of $\inputImage_i$ returned by the algorithm. 
Table~\ref{tab:mcts} gives a comparison with the other two approaches (CW~\cite{CW-Attacks} and JSMA~\cite{JSMA}). The numbers for CW and JSMA are taken from~\cite{CW-Attacks} \footnote{For CW, the $L_0$ distance in \cite{CW-Attacks} counts the number of changed pixels, while for the others the $L_0$ distance counts the number of changed dimensions. Therefore, the number 5.8 in Table~\ref{tab:mcts} is not precise, and should be between 5.8 and 17.4, because colour images have three channels.}, where additional optimisations have been conducted over the original JSMA. According to \cite{JSMA}, the original JSMA has an average distance of 40 for MNIST. 
\begin{table}
\center
\begin{tabular}{|c|c|c|c|c|}
 \hline
$L_0$  & CW ($L_0$ algorithm) & Game (timeout = 1m) & JSMA-F & JSMA-Z \\ 
 \hline 
MNIST &  8.5 & 14.1 & 17 & 20 \\
CIFAR10 & 5.8 & 9 & 25 & 20 \\
 \hline
\end{tabular}
\caption{CW vs. Game (this paper) vs. JSMA}
\label{tab:mcts}
\end{table}
Our experiments are conducted by setting the termination conditions $tc_1 = 20 s$ and $tc_2 = 60 s$ for every image. Note that JSMA needs several minutes to handle an image, and CW is 10 times slower than JSMA~\cite{CW-Attacks}. From the table, we can see that, already in a limited computation time, our game-based approach can achieve a significant margin over optimised JSMA, which is based on saliency distributions, although it is not able to beat the optimisation-based approach CW. We also mention that, in \cite{DLV}, the un-optimised JSMA produces adversarial examples with smaller average $L_2$ distance than FGSM~\cite{FGSM} and  DLV on its single-path algorithm~\cite{DLV}. 
Appendix of \cite{WHK2017} provide illustrative examples exhibiting the manipulations that the three algorithms performed on the images.

\smallskip\noindent{\bf Convergence in Limited Runs}
To demonstrate convergence of our algorithm,
we plot the evolution of three variables related to the adversarial severity $\severity_\inputImage(\inputImage')$ against the number of iterations. 
The variable $best$ (in blue color) is the smallest severity found so far. 
The variable $current$ (in orange) is the severity returned in the current iteration. The variable $window$ (in green) is the average severity returned in the past 10 iterations. The blue and orange plots may overlap because we let the algorithm return the best example when it fails to find an adversarial example in some iteration.
The experiments are terminated with $1/\epsilon$-convergence of different $\epsilon$ value such as 0.1 or 0.05. 
The green plot getting closer to the other two provides empirical evidence of convergence.
\begin{figure}[h]

    \centering
    \includegraphics[width=\textwidth]{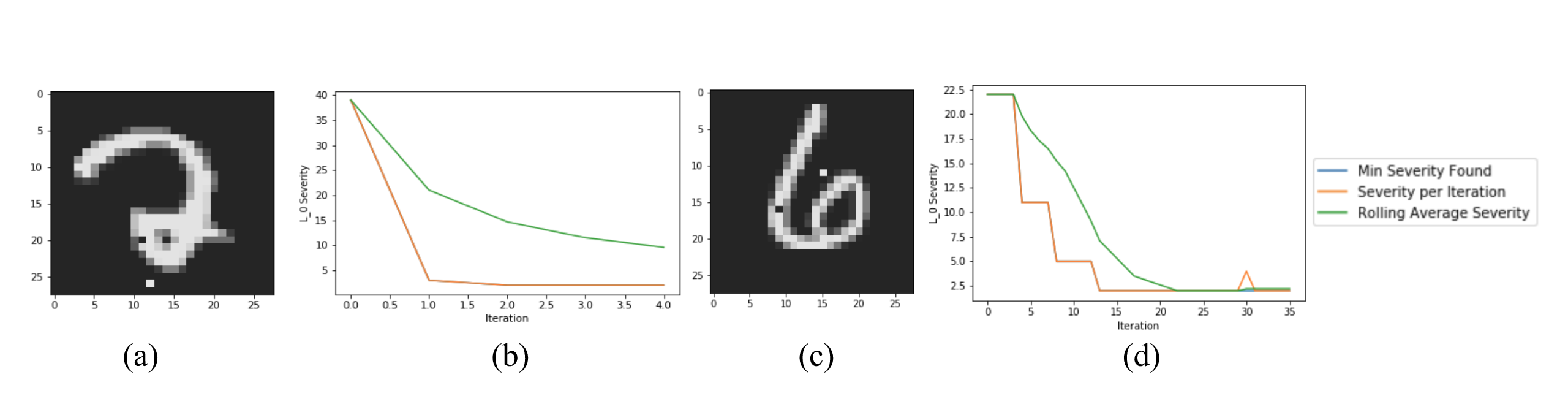}
    
    \caption{(a) Image of a two classified as a seven with 70\% confidence and (b) the demonstration of convergence. (c) Image of a six classified as a five with 50\% confidence and (d) the demonstration of convergence.}
    \label{fig:MNISTConverge}

\end{figure}
In Fig. \ref{fig:MNISTConverge} we show that two MNIST images converge over fewer than 50 iterations on manipulations of 2 pixels, and we have confirmed that they represent optimal strategies of the players. 
We also work with other state-of-the-art networks such as the VGG16 network~\cite{VGG16} from the ImageNet competition.
Examples of convergence are provided in Appendix of~\cite{WHK2017}.

\smallskip\noindent{\bf Evaluating Safety-Critical Networks}
We explore the possibility of applying our game-based approach  to support real-time decision making and testing,  
for which the algorithm needs to be highly efficient, requiring only seconds to execute a task.

We apply our method to a network used for classifying traffic light images collected from dashboard cameras. The Nexar traffic light challenge \cite{NexarData} made over eighteen thousand dashboard camera images publicly available. Each image is labeled either green, if the traffic light appearing in the image is green, or red, if the traffic light appearing in the image is red, or null if there is no traffic light appearing in the image. We test the winner of the challenge which scored an accuracy above 90\% \cite{NexarEntry}. 
Despite each input being 37632-dimensional (112x112x3), our algorithm reports that the manipulation of an average of 4.85 dimensions changes the network classification. 
Each image was processed by the algorithm in 0.303 seconds (which includes time to read and write images), i.e., 304 seconds are taken to test all 1000 images.
\begin{figure}[h]
    \centering
    \includegraphics[width=\textwidth]{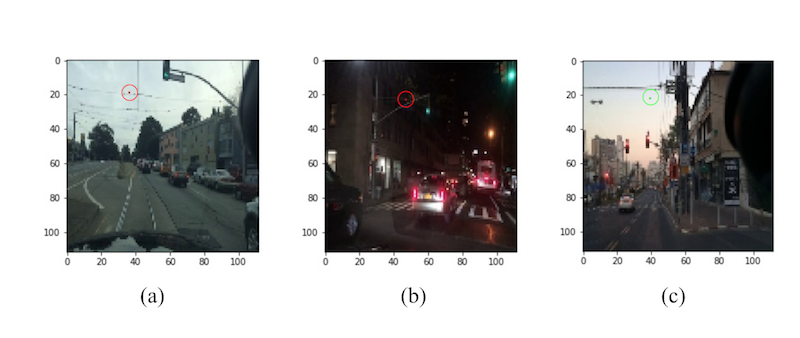}
    \caption{Adversarial examples generated on Nexar data demonstrate a lack of robustness. (a) Green light classified as red with confidence 56\% after one pixel change. (b) Green light classified as red with confidence 76\% after one pixel change. (c) Red light classified as green with 90\% confidence after one pixel change.}
    \label{fig:NexarFig1}
    
\end{figure}
We illustrate the results of our analysis of the network in Fig. \ref{fig:NexarFig1}. Though the images are easy for humans to classify, only one pixel change causes the network to make potentially disastrous decisions, particularly for the case of red light misclassified as green. 
To explore this particular situation in greater depth, we use a targeted safety MCTS procedure on the same 1000 images, aiming to manipulate images into green. We do not consider images which are already classified as green. Of the remaining 500 images, our algorithm is able to change all image classifications to green with worryingly low severities,  namely an average $L_0$ of 3.23.
On average, this targeted procedure returns an adversarial example in 0.21 second per image. 
Appendix~\ref{sec:convergence} provides some other examples.

\begin{figure}[h]
    \centering
    
    \includegraphics[width=0.32\textwidth]{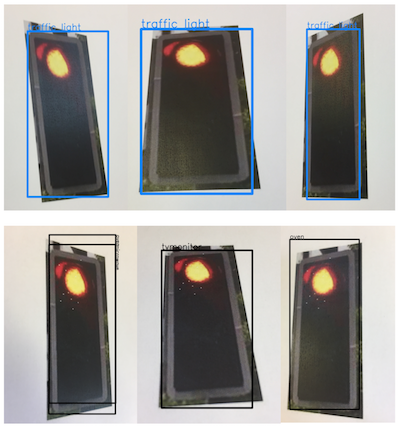}
    ~~~~
        \includegraphics[width=0.34\textwidth]{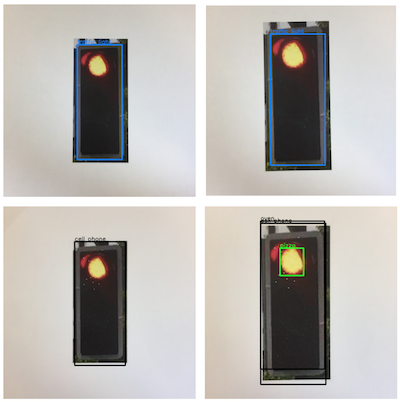}
    \caption{(Left) Adversarial examples in physical domain remain adversarial at multiple angles. Top images  classified correctly as traffic lights, bottom images  classified incorrectly as either ovens, TV screens, or microwaves. (Right) Adversarial examples in the physical domain remain adversarial at multiple scales. Top images correctly classified as traffic lights, bottom images classified incorrectly as ovens or microwaves (with the center light being misclassified as a pizza in the bottom right instance).}
    \label{fig:angleinvarience}
\end{figure}

\smallskip\noindent{\bf Counter-claim to Statements in \cite{NOneedtoworry}} 
A recent paper~\cite{NOneedtoworry} argued that, under specific circumstances, there is no need to worry about adversarial examples because they are not invariant to changes in scale or angle in the physical domain. Our SIFT-based approach, which is inherently scale and rotationally invariant, can easily counter-claim such statements. 
To demonstrate this, we conducted similar tests to \cite{NOneedtoworry}. We set up the YOLO network, took pictures of a few traffic lights in Oxford, United Kingdom, and generated adversarial examples on these images.
For the adversarial example shown in Fig.~\ref{fig:coverimage},  we print and photograph it at several different angles and scales to test whether it remains misclassified. The results are shown in Fig. \ref{fig:angleinvarience}.
In \cite{NOneedtoworry} it is suggested that realistic camera movements -- those which change the angle and distance of the viewer -- reduce the phenomenon of adversarial examples to a curiosity rather than a safety concern. Here, we show that our adversarial examples, which are predicated on scale and rotationally invariant methods, defeat these claims.

\section{Related Works}

We review works concerning the safety (and robustness) of deep neural networks. Instead of trying to be complete, we aim to only cover those directly related.  

\smallskip\noindent{\bf White-box heuristic approaches.}
In \cite{propertiesOfNeuralNetworks}, Szegedy et. al. find a targeted adversarial example by running the L-BFGS algorithm, which minimises the  $L_2$ distance between the images while maintaining the misclassification.
Fast Gradient Sign Method (FGSM) \cite{FGSM}, a refinement of L-BFGS, 
 takes as inputs the parameters $\theta$ of the model, the input $\inputImage$ to the model, and the target label $y$, and computes a linearized version of the cost function with respect to $\theta$ to obtain a manipulation direction. After the manipulation direction is fixed, a small constant value $\tau$ is taken as the magnitude of the manipulation. 
Carlini and Wagner \cite{CW-Attacks} adapt the optimisation problem proposed in \cite{propertiesOfNeuralNetworks} to obtain a set of optimisation problems for $L_0$, $L_2$, and $L_\infty$ attacks. They claim  better performance than FGSM and Jacobian-based Saliency Map Attack (JSMA) with their $L_2$ attack, in which for every pixel $x_i$ a new real-valued variable $w_i$ is introduced and then the optimisation is conducted by letting $x_i$ move along the gradient direction of $\tanh(w_i)$. 
Different from the optimisation approaches, the JSMA \cite{JSMA} uses 
a loss function to create a ``saliency map" of the image 
which indicates the importance of each 
pixel on the network's decision. 
A greedy algorithm is used to gradually modify the most important pixels.
In \cite{FGSM-universal}, an iterative application of an optimisation approach (such as \cite{propertiesOfNeuralNetworks})  is conducted on a set of images one by one to get an accumulated manipulation, which is expected to make a number of inputs misclassified.  
\cite{DBLP:journals/corr/abs-1708-06939} replaces the softmax layer in a deep network with a multiclass SVM and then finds adversarial examples by performing a gradient computation.

\smallskip\noindent{\bf White-box verification approaches.}
Compared with heuristic search approaches, the verification approaches aim to provide guarantees on the safety of DNNs. 
An early verification approach~\cite{PT2010}
 encodes the entire network as a set of constraints. The constraints can then be solved with a SAT solver. 
\cite{KBDJK2017} improves on \cite{PT2010} by handling the ReLU activation functions. The Simplex method for linear programming is extended  
to work with the piecewise linear ReLU functions that cannot be expressed using linear programming. 
The approach can scale up to networks with 300 ReLU nodes. 
In recent work~\cite{GKPB2017} the input vector space is partitioned using clustering and then the method of \cite{KBDJK2017} is used to check the individual partitions. 
DLV \cite{DLV} uses multi-path search and layer-by-layer refinement to exhaustively explore a finite region of  the vector spaces associated with the input layer or the hidden layers, and scales to work with state-of-the-art networks such as VGG16.

\smallskip\noindent{\bf Black-box algorithms.}
The methods in \cite{practical-black-box} evaluate a network  by generating a synthetic data set, training a surrogate model, and then applying white box detection techniques on the model. 
\cite{BlindSearchPaper} randomly searches the vector space around the input image for changes which will cause a misclassification. 
It shows that in some instances this method is 
efficient and able to indicate where salient areas of the image exist.

\section{Conclusion}

In this paper we present a novel feature-guided black-box algorithm for evaluating the resilience of deep neural networks against adversarial examples. Our algorithm employs the SIFT method for feature extraction, provides a theoretical safety guarantee under certain restrictions, and is very efficient, opening up the possibility of deployment in real-time decision support.
We develop a software package and demonstrate its applicability on a variety of state-of-the-art 
networks and benchmarks. 
While we have detected many instabilities in state-of-the-art networks, we have not yet found a network that is safe. 
Future works include comparison with the Bayesian inference method for identifying adversarial examples \cite{DG2017}.


\smallskip\noindent{\bf Acknowledgements}
Kwiatkowska is supported by EPSRC Mobile Autonomy Programme Grant (EP/M019918/1). Xiaowei gratefully acknowledges NVIDIA Corporation for its support with the donation of the Titan Xp GPU, and is partially supported by NSFC (no. 61772232)

\clearpage
\appendix
\section*{Appendix}

The appendix provides further details of our method and experimental results, as well as additional background information, and is organised as follows. In  Section~\ref{sec:convergence}, we provide additional empirical results of our algorithm, while in Section~\ref{sec:proofs} we give proofs of the theorems in Section~\ref{sec:mcts}. Section~\ref{sec:architectures} includes the architectures of the networks used in our experiments. 
SIFT-based feature detection techniques are explained in Section~\ref{sec:SIFT} and the intuition for using such techniques in safety testing is described in Section~\ref{sec:attackintuition}. Section~\ref{sec:nexar} provides information about the Nexar challenge and an overview of the MCTS algorithm is given in Section~\ref{sec:A-MCTS}. 
Finally, Section~\ref{sec:Lk} includes a discussion of the suitability of using $L_k$ distance measures in safety testing and verification and Section~\ref{sec:GeneralDNNs} of the feasibility of generalising our algorithm to DNNs for tasks other than image classification. 

\section{Empirical Results of the MCTS Algorithm}\label{sec:convergence}

In this section, we provide  empirical results showing the performance and convergence of our MCTS algorithm when working on networks trained on MNIST, CIFAR10, Nexar Challenge, and ImageNet datasets. 

\paragraph{MNIST} A few examples showing the convergence of our algorithm on MNIST network are given in Figure~\ref{fig:MNISTConvergeAppendix}. 

\begin{figure}
    \centering
    \includegraphics[width=\textwidth]{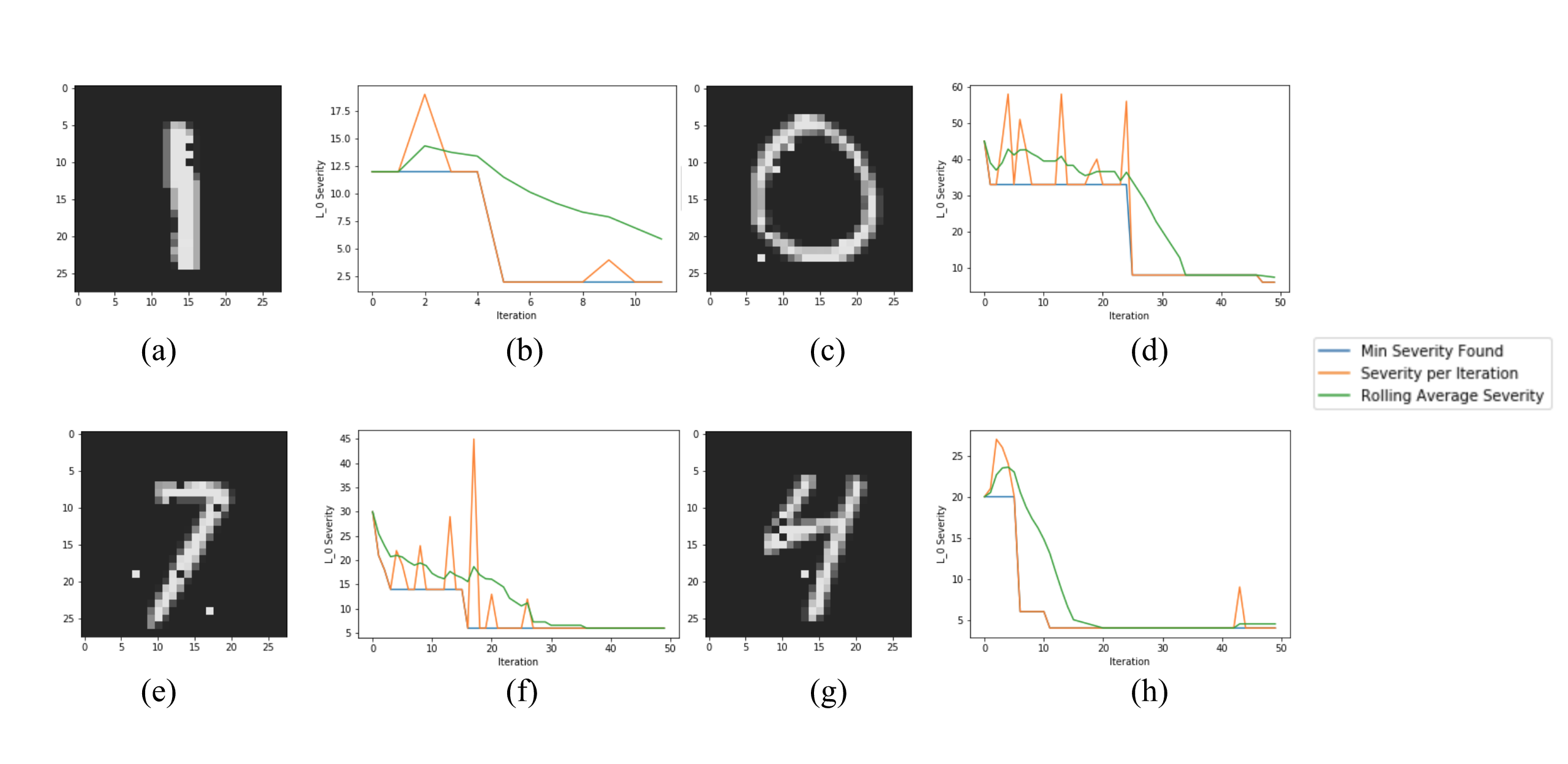}
    \caption{Further empirical evidence of MCTS strategy convergence. (a) Another optimal convergence example, one modified to an eight with confidence 55\%, and (b) plot of MCTS performance over 50 simulations (minimum severity, rolling average severity and severity per iteration). (c) An image of a zero classified as a five with confidence 48\% after six pixel manipulations and (d) MCTS performance on this image. (e) Image of a seven classified as an eight with 47\% confidence after six pixel manipulations and (f) MCTS performance on this image. (g) Image of a four predicted as an eight with 50\% confidence after four pixel manipulations and (h) MCTS performance on this image.}
    \label{fig:MNISTConvergeAppendix}
\end{figure}

\paragraph{CIFAR10} 

Figure~\ref{fig:cifar10} shows the results of a comparison of our algorithm with two state-of-the-art algorithms based on heuristic search, where the CW-$L_0$ attack algorithm is based on gradient descent and the JSMA algorithm on the Jacobian saliency map. 

\begin{figure}

\begin{center}

\hspace{3cm}Game\hspace{2cm}CW\hspace{2.5cm}JSMA

\includegraphics[width=1.45cm,height=1.45cm]{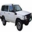}
\includegraphics[width=1.45cm,height=1.45cm]{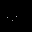}
\includegraphics[width=1.45cm,height=1.45cm]{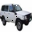}
\includegraphics[width=1.45cm,height=1.45cm]{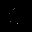}
\includegraphics[width=1.45cm,height=1.45cm]{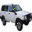}
\includegraphics[width=1.45cm,height=1.45cm]{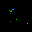}
\includegraphics[width=1.45cm,height=1.45cm]{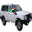}

automobile ~~ $L_0 = 12$~~~~airplane~~~~~~$L_0 = 15$~~~~~~airplane~~~~~~$L_0 = 16$~~~~~~airplane

\includegraphics[width=1.45cm,height=1.45cm]{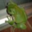}
\includegraphics[width=1.45cm,height=1.45cm]{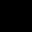}
\includegraphics[width=1.45cm,height=1.45cm]{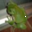}
\includegraphics[width=1.45cm,height=1.45cm]{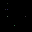}
\includegraphics[width=1.45cm,height=1.45cm]{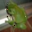}
\includegraphics[width=1.45cm,height=1.45cm]{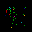}
\includegraphics[width=1.45cm,height=1.45cm]{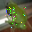}

frog~~~~~~~~~~~~~$L_0 = 3$~~~~~~~~~airplane~~~~~~$L_0 = 18$~~~~~~airplane~~~~~~$L_0 = 40$~~~~~~airplane

\includegraphics[width=1.45cm,height=1.45cm]{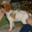}
\includegraphics[width=1.45cm,height=1.45cm]{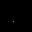}
\includegraphics[width=1.45cm,height=1.45cm]{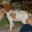}
\includegraphics[width=1.45cm,height=1.45cm]{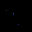}
\includegraphics[width=1.45cm,height=1.45cm]{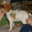}
\includegraphics[width=1.45cm,height=1.45cm]{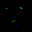}
\includegraphics[width=1.45cm,height=1.45cm]{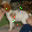}

dog~~~~~~~~~~~~~~~$L_0 = 3$~~~~~~~~~~bird~~~~~~~~~~~$L_0 = 12$~~~~~~~~~bird~~~~~~~~~$L_0 = 10$~~~~~~~~~bird



\includegraphics[width=1.45cm,height=1.45cm]{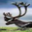}
\includegraphics[width=1.45cm,height=1.45cm]{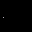}
\includegraphics[width=1.45cm,height=1.45cm]{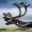}
\includegraphics[width=1.45cm,height=1.45cm]{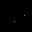}
\includegraphics[width=1.45cm,height=1.45cm]{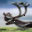}
\includegraphics[width=1.45cm,height=1.45cm]{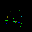}
\includegraphics[width=1.45cm,height=1.45cm]{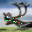}

deer~~~~~~~~~~~~~~$L_0 = 6$~~~~~~~~airplane~~~~~~$L_0 = 6$~~~~~~~airplane~~~~~~~$L_0 = 24$~~~~~~airplane

\caption{Adversarial examples by Game (this paper) vs. CW vs. JSMA for CIFAR-10 networks.}
\label{fig:cifar10}

\end{center}
\end{figure}

\paragraph{Nexar Challenge}

Fig. \ref{fig:NexarFig2} gives some examples of applying our algorithm to \emph{targeted} safety. We also show in Fig. \ref{fig:NexarConverge} that, for many inputs, MCTS is able to find an optimal strategy (a single-pixel misclassification) in only eight simulations (about 0.3 seconds).

\paragraph{ImageNet}

We work with state-of-the-art networks for the imageNet challenge, including the VGG16 network~\cite{VGG16}.
Plots in Figure~\ref{fig:ImageNetConverge} and Figure~\ref{fig:ImageNetConverge2} show clear convergence of our algorithm.

\section{Proofs of Theorems}\label{sec:proofs}

In this section, we provide proofs to the two theorems in Section~\ref{sec:mcts}.

\setcounter{theorem}{2}

\begin{theorem}\label{thm:misclassification}
If all $\tau$-grid images are misclassification aggregators with respect to $\tau/2$, and $\severity(M(\inputImage,k,d), \max_{\strategy_{\playerTwo}}) > d$, then $adv_{N,k,d}(\inputImage,c)=\emptyset$. 
\end{theorem}
\begin{proof} 
First, we need to show that $\eta(\inputImage,k,d) \subseteq \bigcup_{\inputImage_1\in \tauimage(\inputImage,k,d)}\eta(\inputImage_1,1,\tau/2)$. This can be obtained by the definitions of $\eta(\inputImage_1,1,\tau/2)$ and $\eta(\inputImage,k,d)$. 

Now assume that $adv_{N,k,d}(\inputImage,c)\neq \emptyset$. Then there must exist an image $\inputImage'$ such that $\inputImage' \in adv_{N,k,d}(\inputImage,c)$. Because all $\tau$-grid images are misclassification aggregators with respect to $\tau/2$, there must exist a $\tau$-grid image $\inputImage''$ such that $\inputImage'' \in adv_{N,k,d}(\inputImage,c)$. By Theorem~\ref{thm:taugrid}, we have $ \severity_\inputImage(\inputImage'') \geq \severity(M(\inputImage,p,d,\tau), \max_{\strategy_{\playerTwo}})$. By the hypothesis that $\severity(M(\inputImage,k,d), \max_{\strategy_{\playerTwo}}) > d$, we have  $\severity_\inputImage(\inputImage'') > d$, which is impossible because $\inputImage'' \in adv_{N,k,d}(\inputImage,c) \subset \eta(\inputImage,k,d)$. \hfill  $\boxempty$
\end{proof}

\begin{theorem}
Let $N$ be a Lipschitz network with respect to $L_1$ and a constant $\lipschitzConstant $. Then, when $\tau \leq  \frac{2 \classChangeDistance}{ \lipschitzConstant}$ and $\severity(M(\inputImage,k,d), \max_{\strategy_{\playerTwo}}) > d$, we have that $adv_{N,k,d}(\inputImage,c)=\emptyset$. 
\end{theorem}
\begin{proof}
We need to show that $\tau \leq  \frac{2 \classChangeDistance}{ \lipschitzConstant}$ implies that all $\tau$-grid images are misclassification aggregators with respect to $\tau/2$. First of all, by the definition of Lipschitz network, we have $|N(\inputImage_2, N(\inputImage_2))-N(\inputImage_1, N(\inputImage_2))| < \lipschitzConstant \cdot \distance{1}{\inputImage_2 - \inputImage_1} $. Then, by the definition of $\classChangeDistance$, we have $\distance{1}{\inputImage_2 - \inputImage_1} > \classChangeDistance/ h$ when $N(\inputImage_2)\neq N(\inputImage_1)$. Second, we notice that, the statement that all $\tau$-grid images are misclassification aggregators with respect to $\tau/2$ is equivalent to saying that, for any $\tau$-grid image $\inputImage_1$ such that $N(\inputImage_1)=N(\inputImage)$, we have that, for any $\inputImage_2$, $N(\inputImage_2)\neq N(\inputImage_1)$ implies that $\distance{1}{\inputImage_2-\inputImage_1} >  \tau/2$. Finally, we notice that $\distance{1}{\inputImage_2-\inputImage_1} >  \tau/2$ holds when $\distance{1}{\inputImage_2 - \inputImage_1} > \classChangeDistance/ h$ and $\tau \leq  \frac{2 \classChangeDistance}{ \lipschitzConstant}$.  \hfill  $\boxempty$
\end{proof}

\begin{figure}
    \centering
    \includegraphics[width=\textwidth]{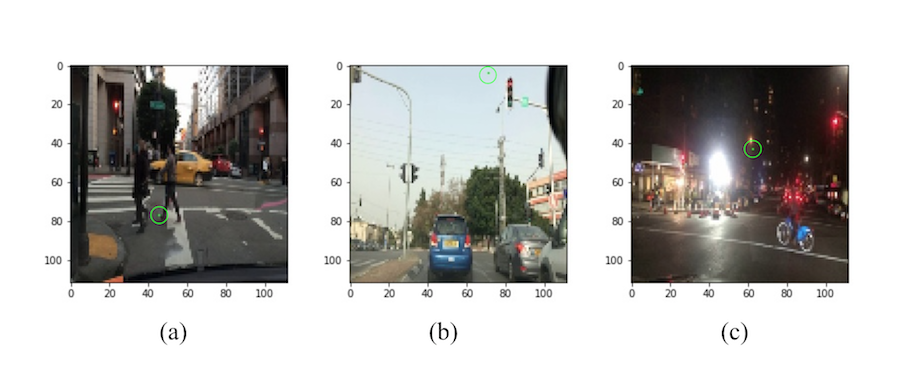}
    
    \caption{Targeted adversarial examples on Nexar  illustrate safety concerns. (a) Red light classified as green with 68\% confidence after one pixel change. (b) Red light classified as green with 95\% confidence after one pixel change. (c) Red light classified as green with confidence 78\% after one pixel change.}
    \label{fig:NexarFig2}
 \end{figure}

 \begin{figure}

    \includegraphics[width=\textwidth]{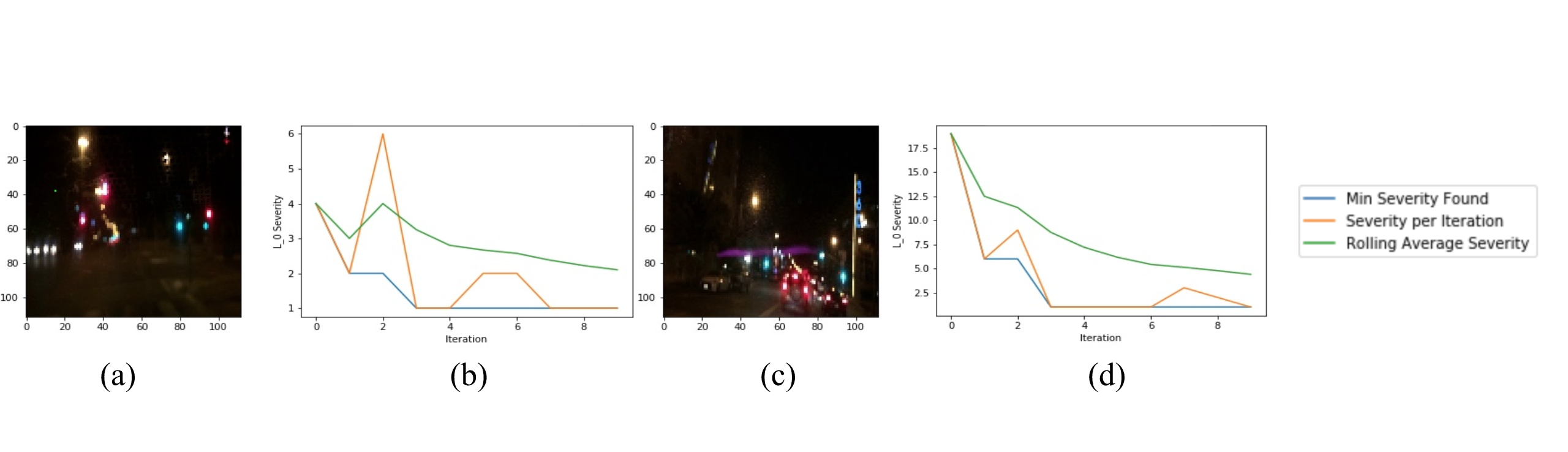}
    
    \caption{Convergence to an optimal strategy on Nexar traffic light images. (a) An image of a red light manipulated into a green light after a single pixel change and the plot of convergence over eight simulations (b). (c) An image of a green light manipulated to a red light after a single pixel manipulation and (d) its convergence plot over eight simulations.}
    \label{fig:NexarConverge}
    
\end{figure}

\begin{figure}
    \centering
    
    \includegraphics[width=\textwidth]{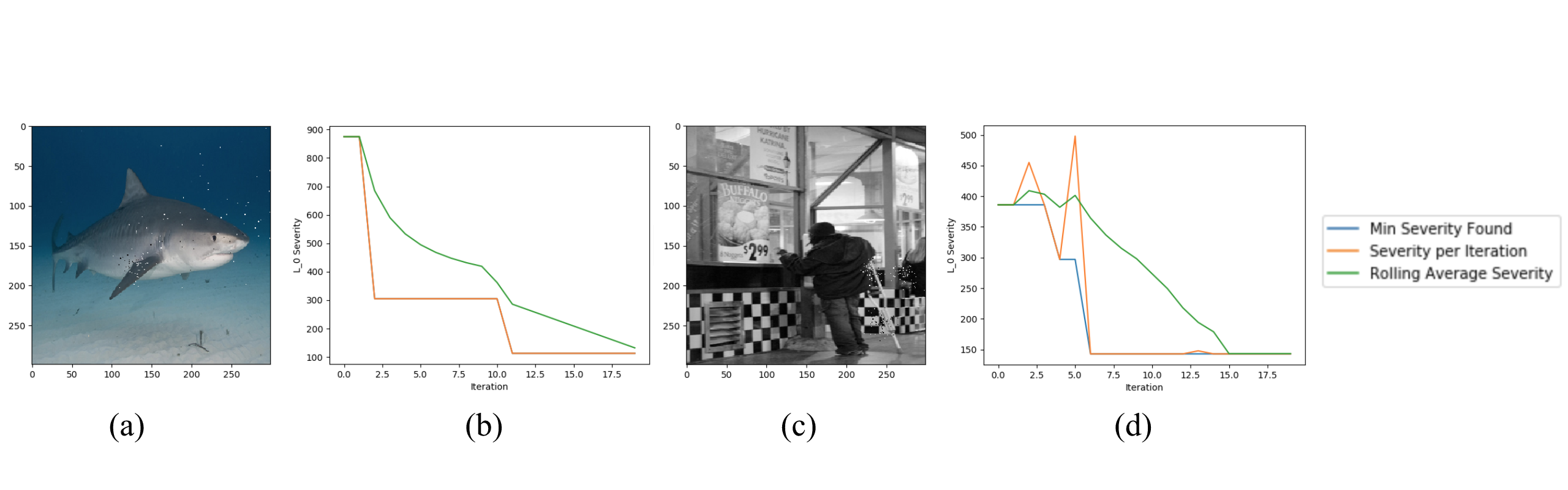}
    
    \caption{Adversarial examples generated on the VGG16 architecture trained on ImageNet data. (a) Image of a great white shark classified as a galeocerdo cuvieri with confidence 42\% after 113 manipulations and (b) the demonstration of convergence over 20 simulations. (b) An image of a crutch classified as bakery after 143 manipulations and (d) the demonstration of convergence 
    over 20 simulations.}
    \label{fig:ImageNetConverge}
    
\end{figure}

\begin{figure}
    \centering
    \includegraphics[width=\textwidth]{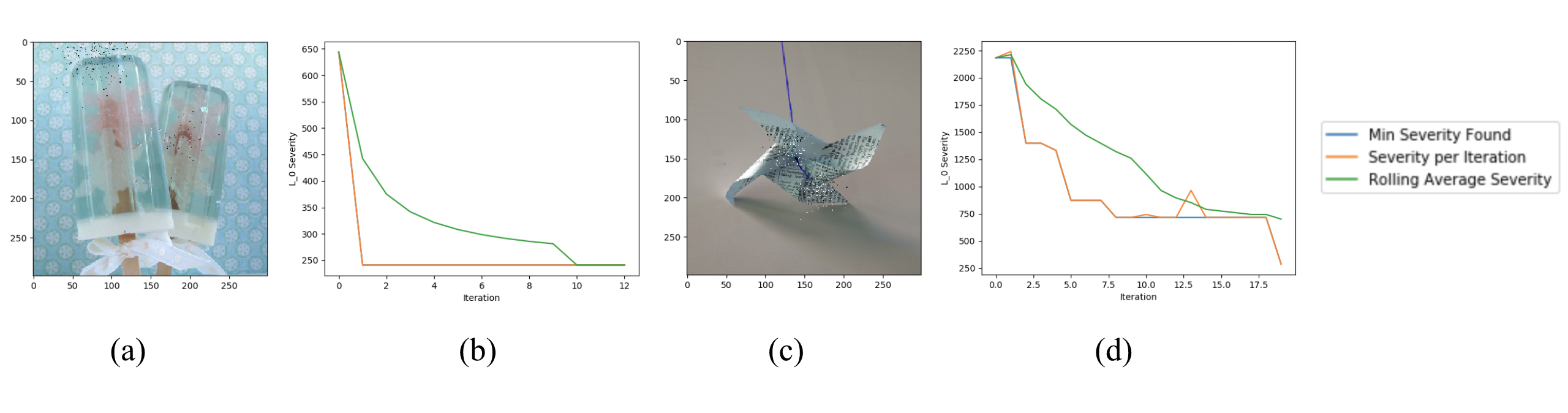}
    \caption{Further empirical evidence of MCTS strategy convergence on state-of-the-art VGG16 network. (a) Image of an ice lolipop predicted as a nipple with 30\% confidence after 241 pixel manipulations and (b) MCTS performance on this image. (c) A pinwheel predicted as a radio telescope with confidence 21\% after 287 pixel manipulations and (d) MCTS performance on this image.}
    \label{fig:ImageNetConverge2}
\end{figure}

\section{Network Architectures in the Experiments}\label{sec:architectures}

This section provides the details of the architectures and training parameters of the networks we work with in our experiments. This includes MNIST networks in Table~\ref{tab:MNISTTable}, CIFAR10 network in Table~\ref{tab:CIFARTable}, VGG16 network for ImageNet dataset in Table~\ref{tab:VGG Arch Table}, and the champion network for Nexar challenge in Table~\ref{tab:NexarTraining}.  

\begin{table}
\begin{center}

 \begin{tabular}{||c c | c c||}
 \hline
 Layer Type & Layer Size & Parameter & SGD \\  
 \hline\hline
 Conv + ReLU & 3x3x32 & Learning Rate & 0.1 \\  
 \hline
 Conv + ReLU  & 3x3x32  & Momentum & 0.9 \\ 
 \hline
 Max Pooling & 2x2 & Delay Rate & - \\ 
 \hline
 Conv + ReLU  & 3x3x64 & Dropout & 0.5 \\  
 \hline
 Conv + ReLU  & 3x3x64  & Batch Size & 128 \\ 
 \hline
 Max Pooling  & 2x2 & Epochs & 50 \\ 
 \hline
 Dense + ReLU & 200 &  &  \\ 
 \hline
 Dense + ReLU & 200 &  &  \\ 
 \hline
 Softmax + ReLU & 10 &  &  \\ 
 \hline
\end{tabular}
\caption{MNIST LeNet Architecture and training parameters used in \cite{CW-Attacks} and \cite{JSMA}.} \label{tab:MNISTTable}
\end{center}
\end{table}

\begin{table}
\begin{center}

 \begin{tabular}{||c c | c c||}
 \hline
 Layer Type & Layer Size (or value) & Parameter & SGD \\  
 \hline\hline
 Conv + ReLU & 3x3x32 & Learning Rate & 0.1 \\  
 \hline
 Conv + ReLU  & 3x3x32  & Momentum & 0.9 \\ 
 \hline
 Max Pooling & 2x2 & Delay Rate & - \\ 
 \hline
 Dropout & 0.25 &  &  \\ 
 \hline
 Conv + ReLU  & 3x3x64 & Dropout & 0.5 \\  
 \hline
 Conv + ReLU  & 3x3x64  & Batch Size & 128 \\ 
 \hline
 Max Pooling  & 2x2 & Epochs & 50 \\ 
 \hline
 Dropout & 0.25 &  &  \\ 
 \hline
 Dense + ReLU & 512 &  &  \\ 
 \hline
 Dropout & 0.5 &  &  \\ 
 \hline
 Softmax + ReLU & 10 &  &  \\ 
 \hline
\end{tabular}
\caption{Architecture and training parameters used in for CIFAR10 dataset.} \label{tab:CIFARTable}
\end{center}
\end{table}

\begin{table}
\begin{center}
 \begin{tabular}{||c c | c c||}
 \hline \hline
 Layer Type & Layer Size & Parameter & SGD \\  
 \hline\hline
 Conv + ReLU & 3x3x64 & Learning Rate & 0.1 \\  
 \hline
 Conv + ReLU & 3x3x64 & Decay & 1e-9\\ 
 \hline
 Max Pooling & 2x2 & Momentum & 0.9  \\ 
 \hline
 Conv + ReLU & 3x3x128 & Nesterov & 1 \\  
 \hline
 Conv + ReLU & 3x3x128 & Loss & Categorical Crossentropy \\ 
 \hline
 Max Pooling & 2x2 &  &  \\ 
 \hline
 Conv + ReLU & 3x3x256 &  &  \\  
 \hline
 Conv + ReLU & 3x3x256 &  & \\ 
 \hline
 Conv + ReLU & 3x3x256 &  &  \\ 
 \hline
 Max Pooling & 2x2 &  &  \\ 
 \hline
 Conv + ReLU & 3x3x512 & &  \\  
 \hline
 Conv + ReLU & 3x3x512 &  &  \\ 
 \hline
 Conv + ReLU & 3x3x512 &  &  \\ 
 \hline
 Max Pooling & 2x2 &  &  \\ 
 \hline
 Conv + ReLU & 3x3x512 &  & \\  
 \hline
 Conv + ReLU & 3x3x512 &  & \\ 
 \hline
 Conv + ReLU & 3x3x512 & &  \\ 
 \hline
 Max Pooling & 2x2 &  &  \\ 
 \hline
 Dense  & 4096 &  &  \\ 
 \hline
 Dropout  & 0.5 &  &  \\ 
 \hline
 Dense  & 4096 &  &  \\ 
 \hline
 Dropout  & 0.5 &  &  \\ 
 \hline
 Dense + Softmax  & 1000 &  &  \\ 
 \hline
\end{tabular}
\caption{Architecture and training parameters of VGG16 for imageNet dataset. \label{tab:VGG Arch Table}}
\end{center}
\end{table}

\begin{table}[H]
\begin{center}
 \begin{tabular}{||c c | c c||}
 \hline \hline
 Layer Type & Layer Size & Parameter & Adam \\  
 \hline\hline
 Conv + ReLU & 3x3x16 & Learning Rate &  3e-4\\  
 \hline
 Max Pooling & 3x3 & Beta 1 &  0.9\\  
 \hline
 Conv + ReLU & 3x3x32 & Beta 2 & 0.999\\  
 \hline
 Max Pooling & 3x3 & Fuzz Factor & 1e-08 \\  
 \hline
 Conv + ReLU & 3x3x64 & Decay & 0.0 \\  
 \hline
 Max Pooling & 2x2 &  &  \\  
 \hline
 Dense & 128 &  &  \\  
 \hline
 Softmax & 3 &  &  \\  
 \hline
\end{tabular}
\caption{Architecture and training parameters for a winning entry in the Nexar Traffic Light challenge \cite{NexarData}.}
\label{tab:NexarTraining}
\end{center}
\end{table}

\section{Feature Detection Techniques}\label{sec:SIFT}



In this section, we give a brief review of a state-of-the-art computer vision algorithm which will be used in our black-box approach. 
%
The Scale Invariant Feature Transform (SIFT) algorithm \cite{SIFT}, 
a reliable technique for exhuming features from an image, makes object localization and tracking possible without the use of neural networks. Generally, the SIFT algorithm proceeds in a few steps: scale-space extrema detection (detecting relatively darker or lighter areas in the image), keypoint localization (determining the exact position of these areas), and keypoint descriptor assignment (understanding the context of the image w.r.t its local area). Below, we give a summary of the SIFT algorithm, and focus on the output and features which will be 
used in our algorithms. 

\paragraph{Scale-Space Extrema Detection}

In \cite{structureOfImages}, it is shown that the only appropriate way to parameterize the resolution of an image without the generation of spurious details (i.e. details which are not inherent in the image, but generated by the method of parameterization) is given by the two-dimensional Gaussian kernel. Lowe \cite{SIFT} uses this to detect extrema in a given image $\inputImage$ by observing the local pixel-value extrema at different scales. Formally, the $k^{th}$ scale of an image $\inputImage$ is calculated\footnote{The SIFT algorithm uses difference of Gaussians at different scales (i.e $S(x,y,k\sigma) -  S(x,y,\sigma)$); for more information see previous work by Lowe \cite{Lowe:1999}.} as follows: 
\[ S(x,y, k\sigma) = G(x,y,k\sigma)*\inputImage(x,y)   \tag{5}\label{eq:5}\]
\noindent
where $x, y$ represent the Euclidean coordinates of a pixel, * is the convolution operator, and $G(x,y,\sigma)$ is the two-dimensional Gaussian kernel given by:

\[ G(x,y,\sigma) = \dfrac{1}{2\pi\sigma^2} exp(-(x^2+y^2)/2\sigma^2) \tag{6}\label{eq:6}.\]

Essentially this parametrization allows us to change $\sigma$ -- the variance of the distribution -- in order to achieve different scales. In practice, it has been noted that, with this parameterization, we are able to filter out some noise of the image, and are able to detect extrema of varying sizes. To get both large and small-sized extrema, we observe the image at a range of scales. Each of the scale ranges is then called an octave, and, after an octave has been calculated, we down-sample the image by a factor of two and observe another octave. In the left images of Fig.~\ref{fig:diffofgaussians}, we show the result of applying the Gaussian kernel to a traffic light. It is clear that this blurring removes some of the unnecessary details within the image and leaves some of the larger features to be examined. After the calculation of a scale space range for each octave, Lowe detects extrema by observing the neighbors of a three by three kernel. If a pixel value is larger or smaller than its neighbors in successive scales, then it is marked as a "keypoint" (as shown in the right portion of Fig.~\ref{fig:diffofgaussians}). 

\begin{figure}[h]
    \centering
    \includegraphics[width=0.8\textwidth]{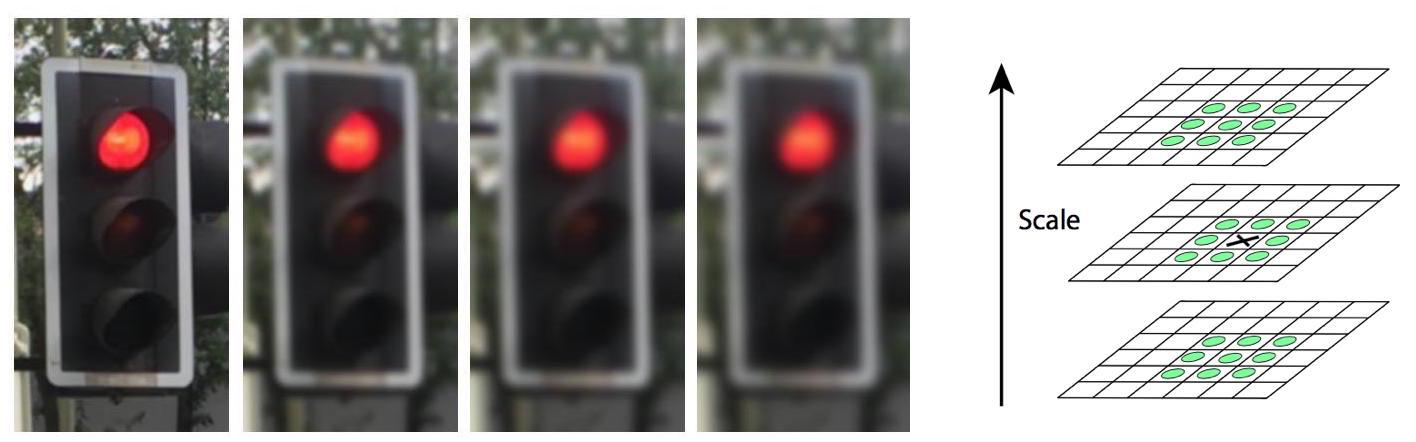}
    \caption{Demonstration of Gaussian blur effect to generate scale. Far right: figure from \cite{SIFT} shows how extrema are detected from these scaled images. Black 'x' marks an extremum if it is larger or smaller than all of the pixels around it.}
    \label{fig:diffofgaussians}
\end{figure}

Importantly, detection of extrema by this method has been shown to be invariant to changes in translation, scaling and rotation, and is minimally affected by noise and small distortions \cite{Lowe:1999}. For our algorithm this means that we should be able to detect and manipulate salient features of images even when the image is of low quality.

\paragraph{Keypoint Description}

After scale-space extrema have been detected, they are located in the original image. Initially, in \cite{SIFT}, this localization was done by translating the pixel location from the scale and octave onto the original image; however, this was later improved by using the Taylor expansion of the scale space function 
shifted so that the origin is at the sampled point. Regardless of which of these is used, the first step of keypoint description is to assign the exact $x$ and $y$ coordinates of the extrema in the image. Once the extrema have been described with a location we refer to them as keypoints in a set $\Lambda$ where each keypoint $\lambda \in \Lambda$ has an $x$ and $y$ coordinate, $\lambda_x$ and $\lambda_y$, respectively. 

After localizing these keypoints, we describe their size and orientation. Size is calculated by the magnitude of the gradient vector corresponding to the keypoint which was located; we will denote the size as  $\lambda_{s}$. After size has been calculated, we sample pixel values from different areas around the keypoint to generate descriptors. The implementation of SIFT in \cite{opencv} (used by our algorithms) gives 128 different local descriptors for each keypoint, which include size, response strength, orientation angle, and local gradient magnitudes. The response strength of keypoints, $\lambda_{r}$, which is derived from the persistence of a keypoint across multiple octaves and scales, is important for our formulation of a salience distribution. 

\begin{figure}[h]
    \centering
    \includegraphics[width=0.5\textwidth]{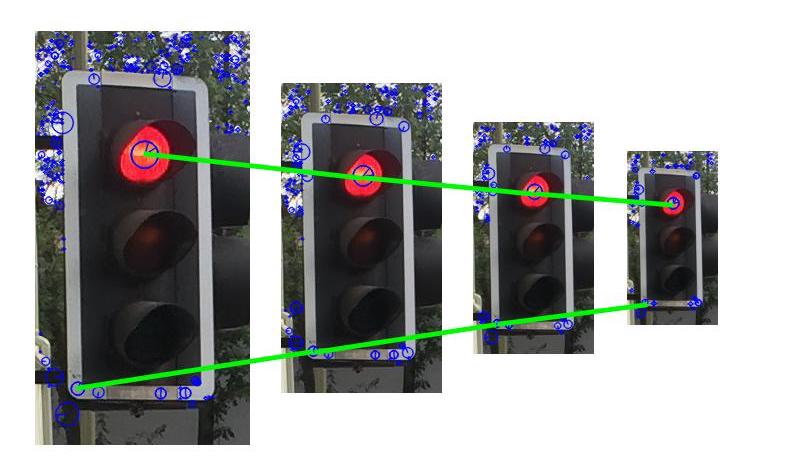}
    \caption{Locating, describing and matching keypoints across the image. Blue circles are keypoints, while green lines represent a few selected keypoints which are persistent throughout each image size.}
    \label{fig:keypoints}
\end{figure}

\section{Intuition for Using Feature Detection for Safety Testing}\label{sec:attackintuition}

It is reasonable to assume that, if any visual system mistakes the classification of an object, then both the spatial and compositional elements of the image must have played an important role. In artificial visual systems, this mapping between an image's basic elements and its classification is systematically learned; however, the ability to know if a system has truly understood the relation between an image's composition or structure and its classification is difficult, and the advent of adversarial examples suggests that artificial visual systems are very sensitive to perturbations of these elements. 

To test this hypothesis -- that an artificial visual system is very sensitive to changes in structural or compositional elements -- we need to be able to pinpoint and manipulate the most important aspects of such elements.

\begin{figure}[h]
    \centering
    \includegraphics[width=\textwidth]{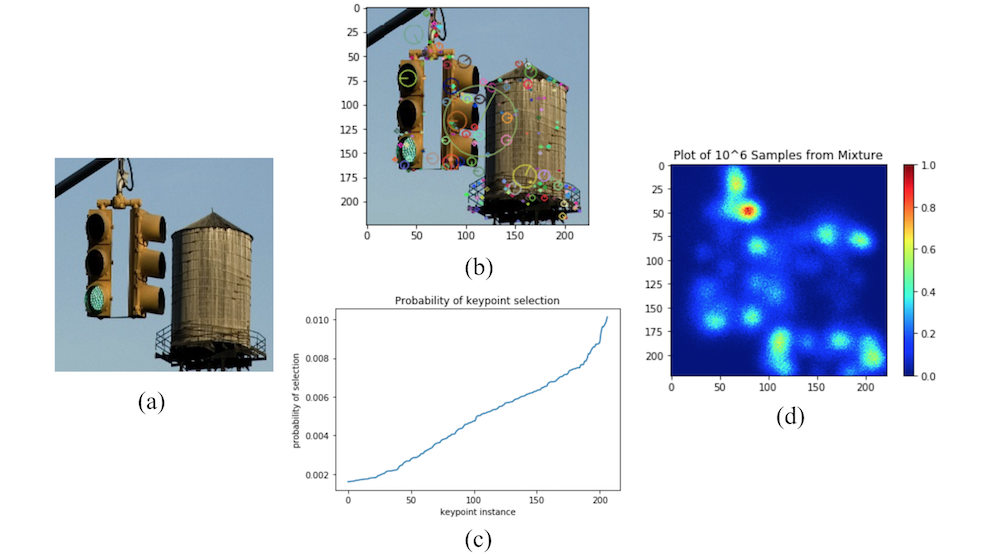}
    \caption{Illustration of the transformation of an image into a saliency distribution. (a) The original image, provided by ImageNet. (b) The image marked with relevant keypoints. (c) The weighting of keypoints by their response strength, where the weights are sorted in order to show the range and distribution of probabilities. (d) An illustration of the probability distribution after observing one million samples.}
    \label{fig:imtosal}
\end{figure}

Over the years, many artificial visual systems have been proposed with varying degrees of success. Modern convolutional neural networks (CNNs) are hypothesized to model the primary visual cortex of humans and primates. The success of modern networks in addition to observations of their hidden layers has been cited as support for this hypothesis. 
Prior to the success of CNNs, feature detection was completed by using methods such as the Scale Invariant Feature Transform (SIFT) algorithm.

The methods by which SIFT computes features is not only deterministic, but well understood as a reliable way to identify basic structural and compositional elements of an image, albeit it may be difficult or expensive to transform this identification to an understanding of an image's classification. CNNs, on the other hand, are able to successfully understand the mapping between an image and it's classification -- but how much of that relationship is dependent on basic structural and compositional elements is unknown. Adversarial examples show that minor changes to such elements can be catastrophic.  

Given the opposing strengths of SIFT and CNNs, it seemed natural to explore the relationship between deterministic feature detection of SIFT and automatic, stochastic feature detection of CNNs. 

\begin{figure}[h]
    \centering
    \includegraphics[width=\textwidth]{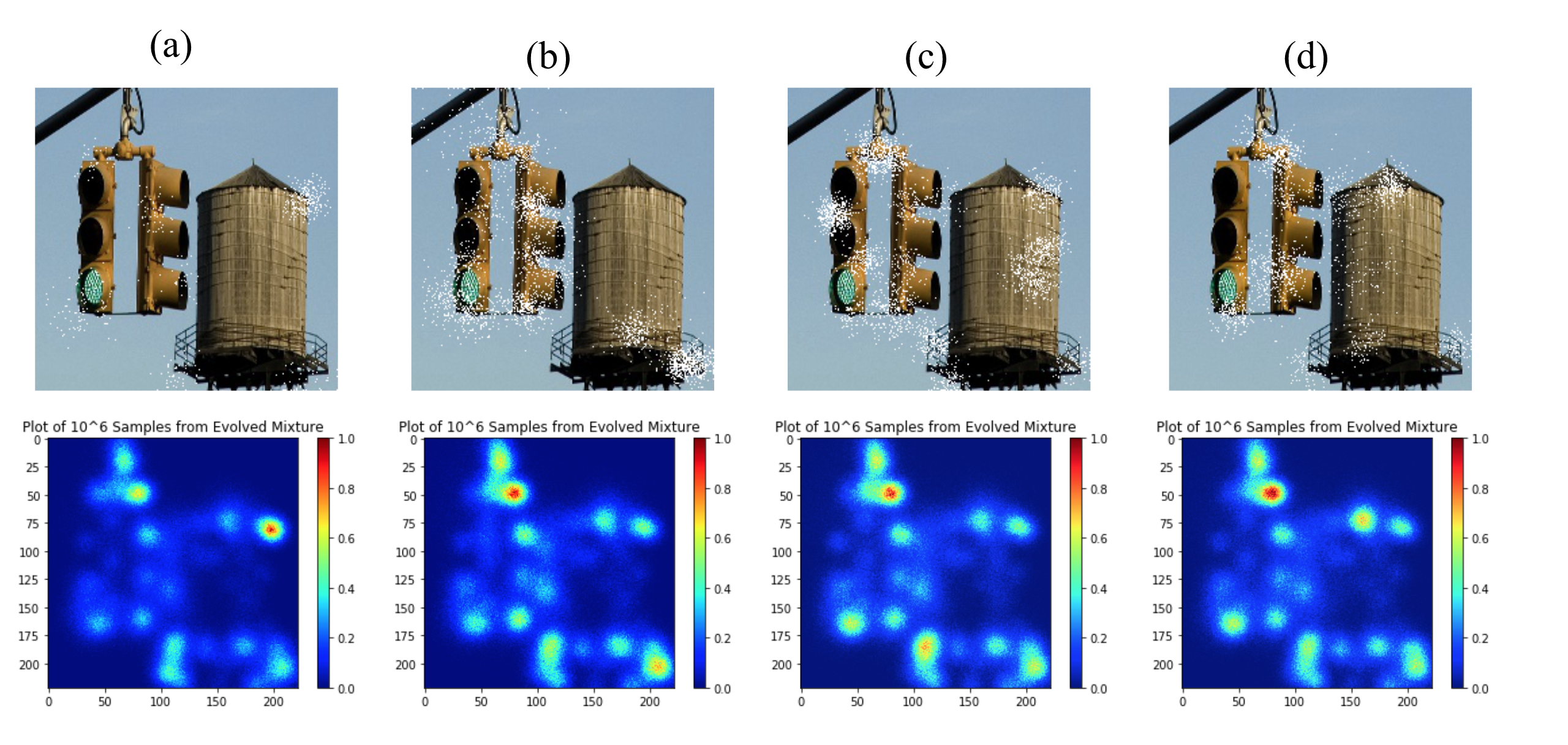}
    \caption{Illustration of the evolution of a saliency distribution with different target classes. Each of the above images was the result of a single, target Monte Carlo simulation with no refinement procedure for illustrative purposes. (a) Modified to Crane with 57.6\% confidence after 572 pixel changes. (b) Modified to Fountain with 48.2\% confidence after 2172 pixel changes. (c) Modified to Castle with 20.4\% confidence after 2553 pixel changes. (d) Modified to Bell with 27.3\% confidence after 1895 pixel changes.}
    \label{fig:evolvedsal}
\end{figure}

In Fig.~\ref{fig:imtosal} we illustrate an image's transformation into a saliency distribution. 
While it is intuitive that these two intimately related algorithms have some common ground, it would not be prudent to assume that the feature detection performed by SIFT and by all CNNs is identical. As such, we introduced Monte Carlo saliency updating and MCTS algorithms to actively re-weight the distribution components based on what we are able to glean from querying the CNN model. Constantly updating the saliency distribution -- see Fig. \ref{fig:imtosal}(d) -- based on the impact of the feature's manipulation wrt the CNN model allows us to dynamically bridge the gap between deterministic and stochastic feature detection (SIFT and CNNs respectively). Further, we can see that, when we allow for the salience distribution to evolve towards a particular target classification, the network weights different keypoints of the image.

Though in Fig.~\ref{fig:evolvedsal} the differences in saliency distributions are subtle, they are also crucial. Flexible saliency distributions are needed to find and manipulate the most crucial elements of an image. For example, a set of ideally placed white pixels can cause the network to believe that a structural element is present where it clearly is not. Similarly, a set of ideally placed pixels can disguise key elements which may lead to a misclassification. In either event, one thing seems certain: the key structural and compositional elements are the heart of any visual system. SIFT provides reliable means to pinpoint these elements, which we exploit in our approach.

Compared to CNNs, SIFT is fast but does not generalise; we do not rely on the latter aspect in our approach.

\section{Information about The Nexar Challenge Network}\label{sec:nexar}

The retraining of the Nexar network  \cite{NexarEntry} -- whose architecture and hyper-parameters are detailed in Table \ref{tab:NexarTraining}
-- can be achieved by executing the script found in the Examples directory of the SafeCV package. Below we give the training details of the network tested in this paper.

\begin{figure}
    \centering
    \begin{minipage}{0.45\textwidth}
        \centering
        \includegraphics[width=0.9\textwidth]{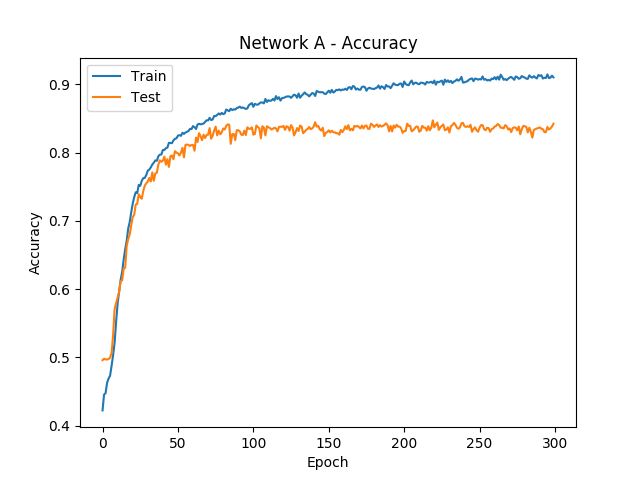} 
        \caption{Accuracy per epoch during the training of the winning entry of the Nexar Challenge \cite{NexarEntry}; training and test accuracy peek at around 90\% and 87\% respectively.}
        \label{NexarAccFigure}
    \end{minipage}\hfill
    \begin{minipage}{0.45\textwidth}
        \centering
        \includegraphics[width=0.9\textwidth]{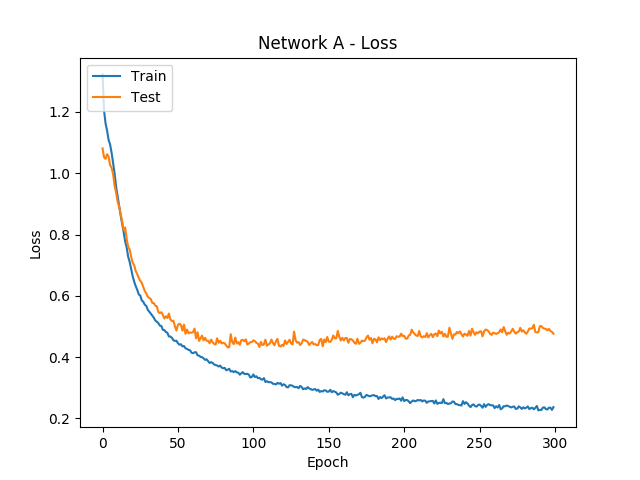} 
        \caption{Loss per epoch during the training of the winning entry of the Nexar Challenge \cite{NexarEntry}.}
        \label{NexarLossFigure}
    \end{minipage}
\end{figure}


\section{Overview of the Monte Carlo Tree Search Algorithm}\label{sec:A-MCTS}

In this section we present a more detailed explanation of the Monte Carlo Tree Search (MCTS) and offer parallels with the game representation of crafting adversarial examples. More detailed information on this method and its origination can be found in \cite{KS2006} and \cite{CWU2008}. Throughout this section we will explain the tree search mechanism first, and then we will reinforce this by recalling the two-player game representation. 

The tree representation is such that each $\alpha'$ that can be reached by one manipulation of $\alpha$ is on the first level, two manipulations of $\alpha$ on the second level, and so on. Further, each node has children that represent the  manipulation of a particular pixel. In the game, a path from the root to a leaf would be modeled with the sequence $s_1u_1s_2u_2...$ of game states such that, for all $k\geq 1$, we have $u_k$ as a move made by player \playerOne\ (selecting which keypoint to manipulated) and $s_{k+1}$ as a move made by player \playerTwo\ (selecting a specific pixel and how to manipulate it). Here, we restrict ourselves to visualizing the moves made by player \playerOne.

During a Monte Carlo tree search, a partial tree of previously explored states is maintained and the policy for exploring the children of each state is continuously updated. Within the framework of our game, this can be thought of as maintaining a series of previous game plays and updating the strategies of each player, respectively. 

Recall that we utilize SIFT keypoints as a high level representation of the input which allows an \textit{a priori} saliency distribution. The components distributions and their weights will serve as initial exploration policies for players \playerOne\ and \playerTwo, respectively, and by extension for exploring the tree. The exploration algorithm is classically split into four steps: \textit{selection}, \textit{expansion}, \textit{simulation}, and \textit{backpropagation}.

\begin{figure}[h]
    \centering
   \includegraphics[width=\textwidth]{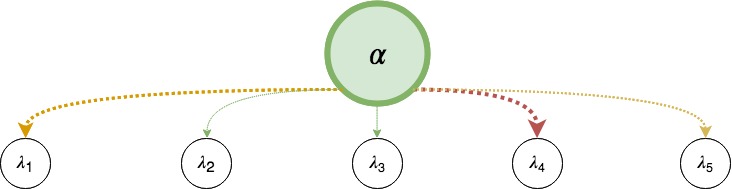}
    \caption{The initial tree with $\alpha$ (the original image) as the root of the tree and each child representing the options of player \playerOne\  (the keypoints of $\alpha$). We use a green color to represent `leaf' nodes, blue to represent internal nodes, dotted lines to represent unexplored paths, solid lines to represent explored paths, and colors to represent probability of selection.}
    \label{fig:mcts-init}
\end{figure}

\subsection{Selection} \label{MCTS:Selection}

Choosing a node to explore is straightforward: the weights of the  Gaussian mixture model provide an initial stochastic policy, which we will call $\hat{\phi}$. 
A distribution derived from these weights defines the the base strategy for player \playerOne. We will cover the update (or evolution) of this strategy in the section on backpropogation. 

After selecting a keypoint, the strategy for selection of a specific pixel, the responsibility of player \playerTwo,  is given by $\GaussianMixture_{i,x}$ and $\GaussianMixture_{i,y}$. We leave the manipulation function purposefully vague so as to not limit approaches that might have different optimization goals.

\begin{figure}[h]
    \centering
   \includegraphics[width=\textwidth]{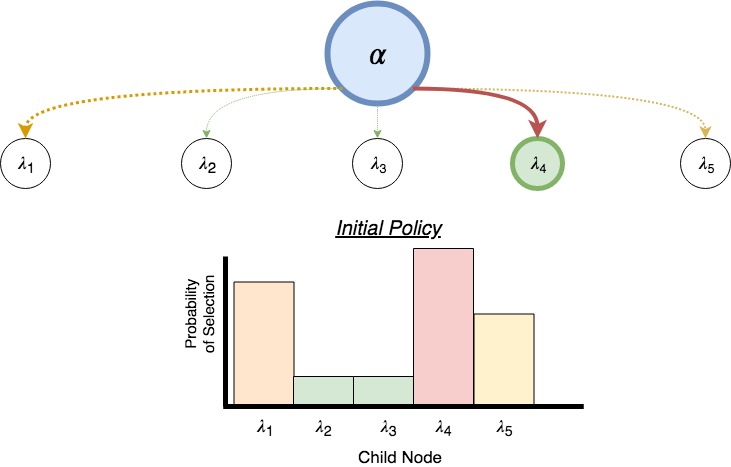}
    \caption{The selection of a node from the root depends on the policy. Here, because the root was previously unexplored this figure not only represents selection but expansion of the root node.}
    \label{fig:mcts-selection}
\end{figure}

\subsection{Expansion} \label{MCTS:Expansion}

Once we reach a node that has not yet been explored (that is, one from which we have not selected before), termed a `leaf' node, we employ the standard exploration strategy, $\hat{\phi}$, to select a new node (this is visualized in Fig.~\ref{fig:mcts-selection} as the \textit{Initial Policy}). Because this child did not previously exist in the tree, this step is known as `expansion.'

Prior to continuing the procedure, the termination conditions are checked.

\subsection{Simulation} \label{MCTS:Simulation}

Once we have selected a new leaf node and made the proper manipulation, we employ Monte Carlo simulation of the manipulation process. That is, we continuously search the tree until a termination condition is met. During this simulation, we exclusively use $\hat{\phi}$ and do not record the nodes we visit as part of the partial, Monte Carlo tree. This could be seen as using the default strategies for each player to arrive at a termination condition. 

\begin{figure}[h]
    \centering
   \includegraphics[width=\textwidth]{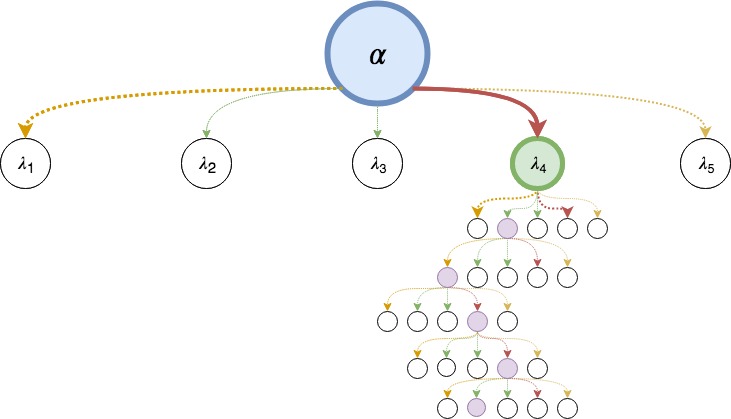}
    \caption{Following the expansion of the root, we continue to select nodes down the tree until a termination condition is satisfied. Note that at each step of the simulation we are using the initial policy and that we have colored the nodes which are not a part of the partial tree in purple.}
    \label{fig:mcts-simulation}
\end{figure}

\subsection{Backpropagation} \label{MCTS:Backpropagation}

Once a termination condition is met, we use the reward function as defined in Section \ref{SafetyThms}. The reward $r$ is then used to update each of the policies in the nodes to the current leaf based on the upper confidence bound (UCB \cite{KS2006}) equation: 

\[ \bar{r}_{j,l} + \sqrt{\dfrac{2ln(n)}{n_j}}\]

\noindent
where $\bar{r}_{j,l}$ is the mean reward after selecting the $j^{th}$ node on the $l^{th}$ level, $n$ is the number of times this nodes parent has been played and $n_j$ is the number of times this node has been selected from all the times its parent has been played. It is clear from inspection that the $j^{th}$ node will have a high confidence bound if the reward of its selection is consistently high or if it is not selected after its parent has been selected many times (i.e. $n >> n_j$).

The turn-based game interpretation of this update function is backtracking through the game play ($s_1u_1s_2u_2...$) and updating each player's strategy for exploring $u_k$ and $s_{k+1}$.
\begin{figure}[h]
\centering
\begin{subfigure}[b]{\textwidth}
    \centering
   \includegraphics[width=\textwidth]{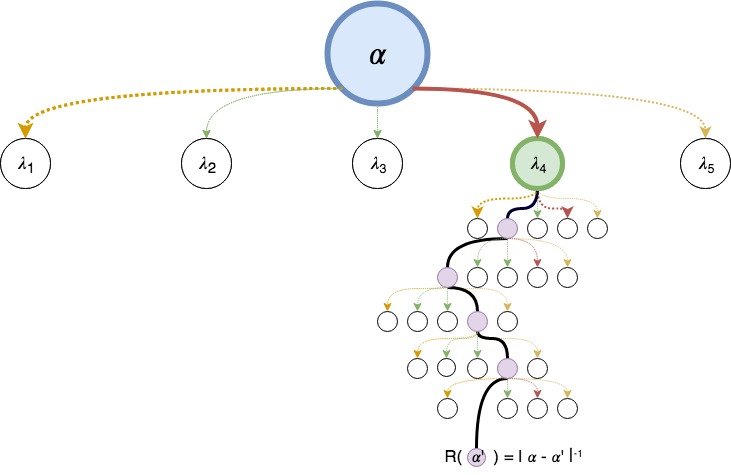}
    \caption{After a termination condition has been met, we calculate the reward of the final node (inverse of severity) and we backpropagate this reward up the tree so that we can update the probabilities of selection from the root node.}
    \label{fig:mcts-backprop}
\end{subfigure}

\begin{subfigure}[b]{\textwidth}
    \centering
   \includegraphics[width=\textwidth]{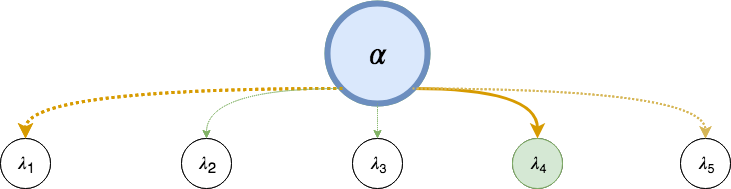}
    \caption{If the adversarial example that was detected required a severe manipulation, then the policy for exploration from the root might discourage future exploration in that direction.}
    \label{fig:mcts-backprop}
\end{subfigure}
\caption{}
\end{figure}

\begin{figure}[h]
    \centering
   \includegraphics[width=\textwidth]{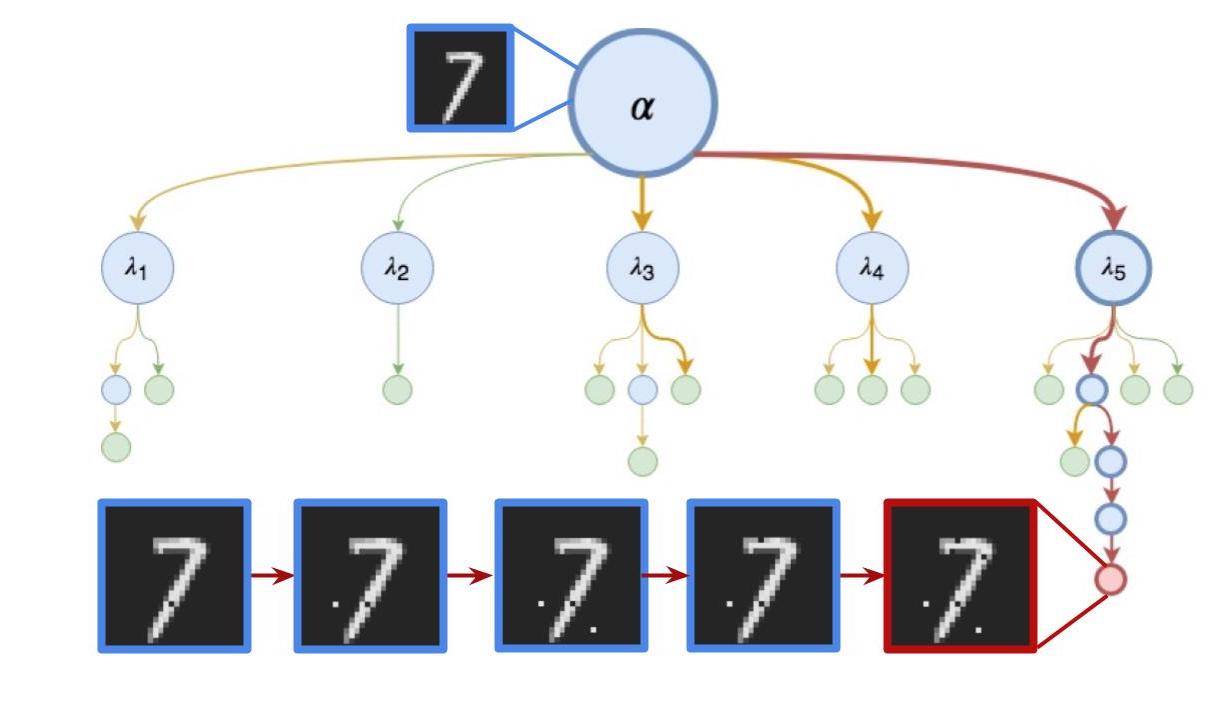}
    \caption{Pedagogical example of asymmetric Monte Carlo tree after a satisfactory adversarial example has been found. We show the detected sequence of manipulations to the adversarial example outlined in blue with red lines between them, and the final adversarial example outlined in red. Note that, though the adversarial example and the sequence of manipulations are empirical, the tree represented here is contrived for pedagogical purposes.}
    \label{fig:mcts-full}
\end{figure}


\section{$L_k$ Distance as a Manipulation Severity Metric}  \label{sec:Lk}

In general, accurately quantifying the perceptual similarity between distinct images remains a problem without a perfect solution. Accurate encapsulation of perceptual differences is central to any procedure that seeks an optimal (i.e. minimally distinct) adversarial example. In the absence of a perfect metric for measuring perceptual similarity, the $L_k$ norm has been widely adopted. Unfortunately, no one value of $k$ can be selected on principle. In this section we describe how changing the parameter $k$ alters the optimization procedure in the context of finding perceptually minimal examples.

As we can see in Fig.~\ref{fig:lkdistance}, when we speak of the ``severity" of a manipulation, it is imperative that this be thought of in the context of the particular $k$ that has been chosen, as choosing different $k$ for each image will vary the perceived severity significantly. For example, selecting small $k$ (e.g. 0 or 1) optimizes for sparse but intense changes to single pixel values; on the other hand, selecting large $k$ values optimizes for pervasive but slight changes to pixels the image. 

\begin{figure}[H]
    \centering
    \includegraphics[width=\textwidth]{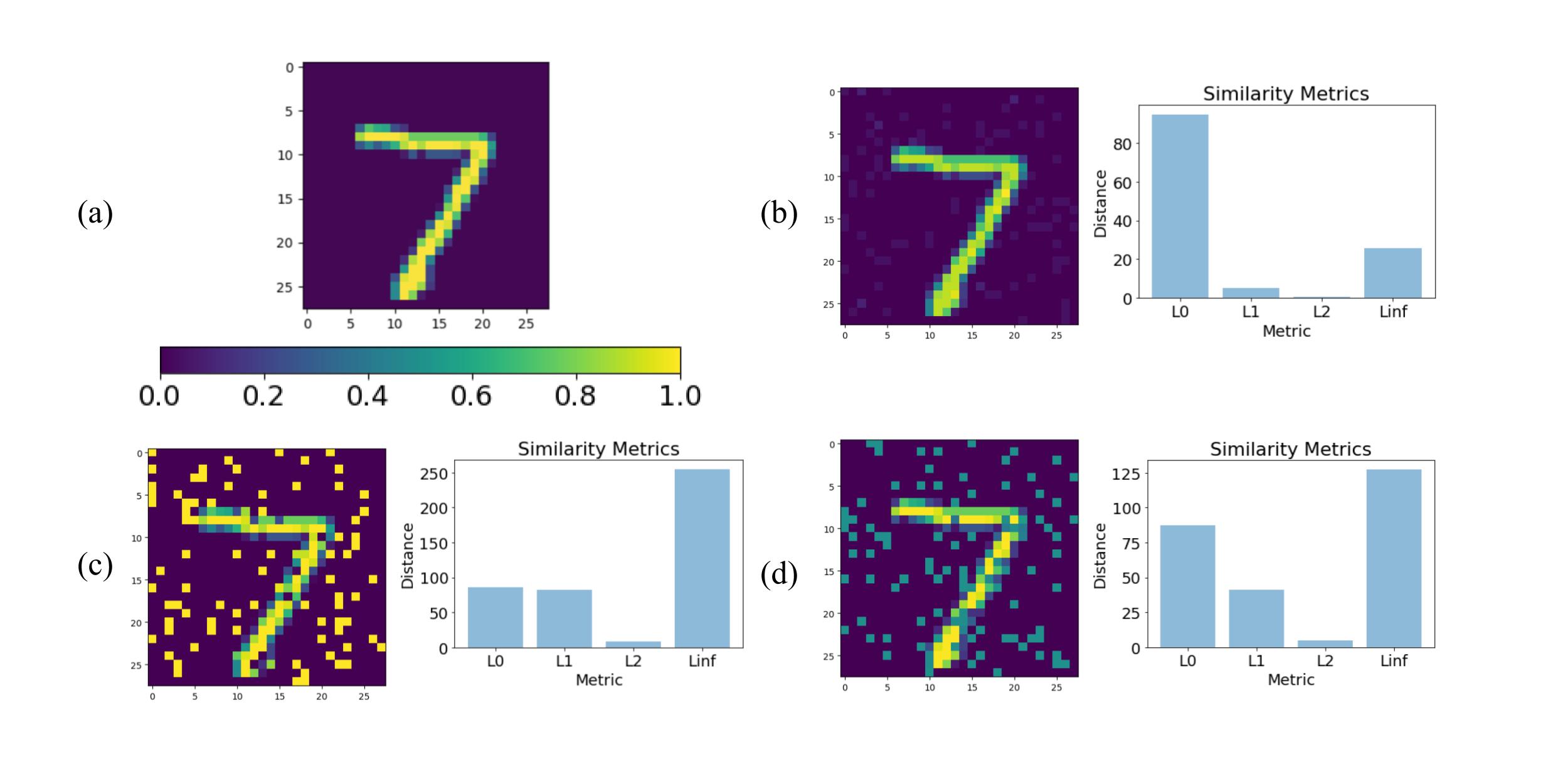}
    \caption{Comparison of $L_k$ distance metrics. (a) We can get an idea of how intense each pixel is in the original image. (b) Many pixels manipulated only slightly leads to high $L_0$ and low $L_\infty$. (c) Half the pixels manipulated in (b) but each pixel disturbed more from its original value leads to lower $L_0$ than $L_\infty$. (d) Manipulating fewer values than in (b) but more than (a) and with an intermediate intensity change exacerbates the distance between $L_0$ and $L_1$.}
    \label{fig:lkdistance}
\end{figure}

\section{Extension to General Purpose Deep Neural Networks}  \label{sec:GeneralDNNs}

The methods presented in this paper focus on the image domain, a mature application of neural networks (and one that has been proposed for safety and security-critical applications). In this section, we describe how a similar approach could be used to test neural networks with more diverse inputs. 

Given a network, and a way to semantically partition that input dimensions into subsets (used as substitutes for keypoints), one can apply the method described in this paper. Using any regular partition of the dimensions of the input data will allow for the algorithm to evolve the saliency values to arrive at a coarse-grained approximation of features in the input.


This points to the issue of how to assign initial weights in order to accurately encapsulate saliency. One advantage to the MCTS algorithm is that it will actively learn saliency even if it is supplied with a random distribution (though this will slow down convergence to an optimal strategy). Often, however, it is the case that machine learning practitioners have some domain knowledge that will allow for an \textit{a priori} assignment of saliency, even if that assignment is a very rough estimate.


\end{document}